\documentclass[letterpaper]{article}
\usepackage{uai2020}
\usepackage[margin=1in]{geometry}
\usepackage{resizegather}

\usepackage{times}

\usepackage{amsmath,amsthm,amssymb,mathtools,amsfonts,bbm,nicefrac,xfrac}
\usepackage{graphicx,caption,color,comment,hyperref,boldline,color,pifont,exscale,booktabs}
\usepackage[boxed]{algorithm2e}
\usepackage{natbib}
\usepackage{resizegather}
\allowdisplaybreaks

\DeclareMathOperator{\R}{\mathbb{R}}

\newcommand{\bie}{\textsc{WI-UCB}}
\newcommand{\bieboth}{\textsc{WI/WO-UCB}}
\newcommand{\ucbr}{\textsc{UCB-Revisited}}
\newcommand{\ucb}{\textsc{UCB1}}
\newcommand{\ucbtwo}{\textsc{UCB2}}
\newcommand{\moss}{\textsc{MOSS}}
\newcommand{\aae}{\textsc{SE}}

\newcommand{\narms}{K}
\newcommand{\arms}{\mathcal{K}}
\newcommand{\indicator}{\mathbb{I}}
\newcommand{\armsm}{\mathcal{K}_m}
\newcommand{\armsmp}{\mathcal{K}_{m+1}}
\newcommand{\estmeanmj}{\overline{X}_{m,j}}

\newcommand{\expd}{\mathbb{E}[D]}
\newcommand{\expz}{\mathbb{E}[Z]}
\newcommand{\expdj}{\mathbb{E}[D_j]}
\newcommand{\expzj}{\mathbb{E}[Z_j]}

\newtheorem{innercustomgeneric}{\customgenericname}
\providecommand{\customgenericname}{}
\newcommand{\newcustomtheorem}[2]{%
  \newenvironment{#1}[1]
  {%
   \renewcommand\customgenericname{#2}%
   \renewcommand\theinnercustomgeneric{##1}%
   \innercustomgeneric
  }
  {\endinnercustomgeneric}
}

\newcustomtheorem{customthm}{Theorem}
\newcustomtheorem{customlemma}{Lemma}
\newtheorem{lemma}{Lemma}
\newtheorem{theorem}{Theorem}
\newtheorem{corollary}{Corollary}
\newtheorem{definition}{Definition}
\newtheorem{remark}{Remark}

\title{Learning by Repetition: Stochastic Multi-armed Bandits under Priming Effect}

\author{ {\bf Priyank~Agrawal }\\
University of Illinois at Urbana-Champaign\\
\href{mailto:priyank4@illinois.edu}{priyank4@illinois.edu}\\
\And
{\bf Theja~Tulabandhula}  \\
University of Illinois at Chicago\\
\href{mailto:theja@uic.edu}{theja@uic.edu}
}

\begin{document}

\maketitle

\begin{abstract}
We study the effect of persistence of engagement on learning in a stochastic multi-armed bandit setting. In advertising and recommendation systems, repetition effect includes a wear-in period, where the user's propensity to reward the platform via a click or purchase depends on how frequently they see the recommendation in the recent past. It also includes a counteracting wear-out period, where the user's propensity to respond positively is dampened if the recommendation was shown too many times recently. Priming effect can be naturally modelled as a temporal constraint on the strategy space, since the reward for the current action depends on historical actions taken by the platform. We provide novel algorithms that achieves sublinear regret in time and the relevant wear-in/wear-out parameters. The effect of priming on the regret upper bound is also additive, and we get back a guarantee that matches popular algorithms such as the UCB1 and Thompson sampling when there is no priming effect. Our work complements recent work on modeling time varying rewards, delays and corruptions in bandits, and extends the usage of rich behavior models in sequential decision making settings.

\end{abstract}

\section{INTRODUCTION}\label{sec:introduction}

In advertising applications and recommendation systems, there has been a large body of work that models consumer behavior~\citep{hawkins2009consumer,solomon2014consumer}. One such effect that is relatively well studied is the priming/repetition effect. Under the priming effect, an advertiser's payoff (for instance, click through rate) depends on how frequently they have presented the same ad to the same audience in the recent past. If the advertiser presents a specific ad sporadically, then the click through rate is much lower, even if this ad is the best among a collection of ads. Priming can be broken down into two sub-effiects; \emph{wear-in} and \emph{wear-out}~\citep{pechmann1988advertising}. Wear-in effect leads to a user not responding to an ad if it has not been shown enough number of times in the recent past. Whereas, the wear-out effect leads to a user not responding (or becoming insensitive) to an ad if it has been shown too many times in the recent past. Different ads may need different levels of repetition to obtain payoffs, and all the relevant parameters that model the priming effect may not be known a priori to the advertiser~\citep{ma2016user}. 

This phenomenon also translates to recommendations, such as for products and movies, where repeated display of item(s) can cause positive reinforcement to build over time, culminating in a conversion. It can also lead to fatigue and therefore no conversion. Since these conversion events depend on the past recommendations, they interfere with learning the true underlying (mean) payoffs and demands of different recommendations. 
Motivated by the above discussion, we define a new class of problems, which we call \emph{bandit learning under priming effect}, to address the agent's need for repetitions. The amount by which the agent (say an ad platform or a recommendation system) needs to replay an arm depends on the degree of priming effect (see Section~\ref{sec:problem_definition} for a formal treatment). In essence, the platform's current rewards are functions of its previous actions. A diagram illustrating this is shown in Figure~\ref{fig:1}.  Our model and solution is one among a growing literature~\cite{kveton2015cascading,den2017dynamic,shah2018bandit,wang2019thompson} that focuses on combining empirically validated behavioral models with sequential decision making.

\begin{figure}
\centering
\includegraphics[width=.5\textwidth]{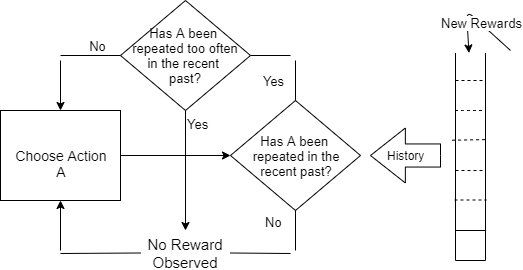}
\caption{Bandit learning under priming effect.}
\end{figure}\label{fig:1}

While our setting can be addressed using Markov Decision Processes/reinforcement learning (RL) techniques, we choose to use the framework of multi-armed bandits (MAB) for their simplicity, analytical tractability and relatively tighter regret guarantees. The MAB problem, which is a special case of the RL problem, captures exploration-exploitation trade-off in certain sequential decision making settings. 
In a stochastic MAB set-up, reward distributions are associated with each arm of the set of arms $\arms$. When the algorithm plays an arm, it immediately receives a reward with which it can learn, and also suffers a regret, which is the difference between the obtained reward versus the reward that it could have obtained, had it played the best arm in hindsight.

In the presence of priming effect, popular MAB algorithms such as \ucb~\citep{auer2002finite}, \aae~\citep{even2006action}, \moss{}~\citep{audibert2009minimax} or Thompson Sampling~\citep{chapelle2011empirical} can be ineffective because they cannot directly control the number of times an arm is played in any given time window. For problem instances with non-unique optimal arms, the aforementioned algorithms may in fact switch between these very frequently (see Appendix~\ref{sec:experiments}), potentially causing linear regret. Even when the mean rewards are not close, initial exploration will yield no rewards due to wear-in effects, hampering learning and hence, subsequent exploitation. To address these issues, we develop \bie and \bieboth{}, which expand on the phase-based algorithmic template~\citep{auer2010ucb} to learn in the presence of priming effect. 

The consequence of priming effect considered here is closely related to the recent works in corruption~\citep{lykouris2018stochastic} and delay~\citep{pike2018bandits}  in the reward accrual process. Unlike these settings, priming effect is \textit{endogenous}, and correlates past actions with the current reward. As an unifying view, in all three works one can assume that there is a intermediate function that allows an MAB algorithm to accrue some transformation of the current and past rewards instead of just the current reward. In our case, we accrue rewards that are modulated by the stochastic sequence of actions that our algorithm took previously, which makes learning more challenging as rewards are now policy dependent.


For instance, \cite{lykouris2018stochastic} consider settings with arbitrary exogenous corruptions of rewards, and propose a randomized algorithm that achieves smooth degradation as the corruption level increases. Unlike their setting, priming effect is \textit{endogenous}, and correlates past actions with the current reward.
While the amount of corruption due to priming in our setting would be O($N K\log T$) ($N$ is a instance dependent parameter in our setting, $K$ is the number of arms, and $T$ is the horizon length) if our algorithm is used, we cannot reuse their analytical techniques because of endogeneity.
Impact of delayed rewards on learning has been well-studied recently ~\citep{joulani2013online,perchet2016batched,cesa2018nonstochastic}. In particular, ~\cite{joulani2013online} provide a recipe to use any regular MAB algorithm in this setting and show that delay causes an additive regret penalty. In~\citep{cesa2018nonstochastic} for adversarial bandits and in~\citep{pike2018bandits} for stochastic bandits, the authors provide regret guarantees for a much weaker setting where rewards can get mixed up and may partially accrue over time. That is, components of rewards due to multiple previous actions may appear collectively at some future point. In contrast to the delay effect, which is affecting the future accrual of rewards, the priming/repetition effect can be viewed as being caused by the trajectory of past actions. As a result, we obtain very different regret bounds.

There has been recent parallel work on rotting bandits~\citep{levine2017rotting,seznec2019rotting}, which can be thought of as capturing the wear-out effect via a sequence of reward random variables with decreasing means. While it does not capture the wear-in effect, the dependence of the reward means on the number of times the ad/arm is played is accounted from the start of the horizon rather than an immediate preceding window, which is critical in our applications. In a closely related subsequent work~\citep{pike2019recovering}, the authors make the mean reward of each arm an unknown function of the time since its last play. While this is an interesting structure, it does not again capture repetition effects (wear-in and wear-out) studied here, and the work relies on a very different modeling setup (using Gaussian processes) to obtain regret bounds.

Some prior works have studied bandit settings under rich temporal user behavior models. For instance, \cite{shah2018bandit} study a temporal behavioral effect rooted in microeconomic theory, namely self-reinforcement. In addition to delays and corruptions influencing rewards, works such as ~\citep{xu2018reinforcement} and~\citep{gamarnik2018delay} also consider the impact of reward accrual in the presence of limited memory, which affects learning and regret. Finally, bandits with switching costs~\cite{banks1994switching,dekel2014bandits}  consider penalties for switching arms too often. As we will discuss soon, our algorithms are also candidate solutions to this problem setting by virtue of switching arms rarely. This is because priming effect imposes a hard constraint on switching arms too frequently and too infrequently, which is approximately equivalent to having large switching costs.


\noindent\textbf{Our Results and Techniques:} Owing to the nature of the priming effect, algorithms necessarily have to ensure that the arms still under consideration are played frequently, perhaps in batches. Phase based algorithms form a natural algorithmic template for such a mechanism. This family of algorithms date back to ~\citep{agrawal1988asymptotically}, who considered arm switching costs. Our algorithms, \bie{} and \bieboth, follow this design pattern, wherein the focus is to eliminate arms between stages (for instance, in \bie{}  each arm is played consecutively for multiple rounds between stages). In particular, both \bie{} (wear-in effect setting) and \bieboth{} (wear-in and wear-out effect setting) are based on algorithms such as \ucbr{}~\citep{auer2010ucb} and \ucbtwo{}~\citep{auer2002finite}, and work under the setting when just the expected priming effect parameters are known. We also introduce a key new idea of \emph{compound-arms} in \bieboth, which lets us retain the algorithmic structure described above as well as the corresponding analytical machinery to obtain regret bounds. 

In our analysis, we design martingales on the sequence of the cumulative sums of the accrued reward deviations from their means. Following the techniques of~\citep{pike2018bandits} and~\citep{auer2010ucb}, we use foundational tools, such as the Bernstein inequality for martingales and the Azuma-Hoeffding inequality based Doob's optimal stopping theorem (see Sections~\ref{sec:wear_in} and \ref{sec:wiwo}), to bound the priming effect under a judicious choice of phase lengths, and guarantee fast convergence of the reward estimates with high probability. Our analysis deviates from these previous works in the following ways: (a) our reward random variables are functions of the past history and \textit{policy} dependent, (b) the use of a phase-based strategy is only possible due to the notion of compound-arms (novel to this work and different from the dueling bandit literature~\citep{yue2012k}), and (c) their regret analysis is not directly applicable to our setting.


The key technical challenge in our setting is due to the opposing wear-in and wear-out effects: while wear-in requires frequent repetition to learn, wear-out hampers learning if arms are played too frequently. Assume that an algorithm can accrue a reward for pulling arm $j$ at time $t$ if it has been tried at least $D_{t,j}$  and at most $Z_{t,j}$ times in the past $N$ rounds, where $D_{t,j}$ and $Z_{t,j}$ are unobserved random variables with means $\expdj$ and $\expzj$, and $N$ is a known instance-specific fixed positive integer. Then, the regret upper bounds of our algorithms, shown in Table~\ref{tab:introresults}, depend \emph{sublinearly} and \emph{additively} on the priming effect parameters (for simplicity assume $\expdj = \expd$ for all $j$). To our knowledge, this is the first work that takes into account both wear-in and wear-out effects for advertising and recommendation systems in an online learning scenario.

\begin{table}
    \centering
\resizebox{\columnwidth}{!}{
\begin{tabular}{ |c||c|}
\hline
Algorithm   &   Bound \\
\hline
Lower bound~\citep{lai1985asymptotically}    &   $\mathrm{O}(\sqrt{KT})$\\
\ucb~\citep{auer2002finite}     &   $\mathrm{O}(\sqrt{KT\log T })$   \\
\moss~\citep{audibert2009minimax}    &   $\mathrm{O}(\sqrt{KT})$\\
\bie{ }[\textit{this work}, no priming]    & $\mathrm{O}(\sqrt{KT\log T })$ \\
\bie{ }[\textit{this work}, wear-in only] & $ \mathrm{O}\left(\sqrt{KT\log T} + K\sqrt{\log^2 T \expd}\right)$\\
\bieboth{ }[\textit{this work}, wear-in \& wear-out] & $\mathrm{O}\left(K\sqrt{T\log T} + K^2\sqrt{\log^2 T N\expd}\right)$\\
\hline
\end{tabular}\label{tab:results}
}
\caption{Results summary.}
\label{tab:introresults}
\end{table}

\section{PROBLEM DEFINITION}\label{sec:problem_definition}

There are $\narms > 1$ arms in the set $\arms$ available to the platform (agent/learner), each corresponding to an ad/recommendation. An arm is played by the platform in each of the $T>1$ rounds of interaction with the user (environment).  

\noindent{\textbf{Priming Effect:}} Each arm $j \in \arms$ is associated with a reward distribution $\xi_j$, which has support in $[0,1]$. The mean reward for arm $j$ is $\mu_j$. $\mu^*$ is the maximum of all $\mu_j$ and corresponds to the arm $j^*$. Let $N\,\in\,\mathbb{N}$ denote the number of historical rounds, and let $\xi_j^D$ and $\xi_j^Z$ for each $j \in \arms$, with supports in $\{0,...,a\}$ and $\{b,...,N\}$ respectively, parameterize the priming effect (wear-in and wear-out). In particular, the distributions $\xi_j^D$ and $\xi_j^Z$ (where $0\leq a<b\leq N$ are fixed non-negative integers) are associated with the wear-in and wear-out effects respectively, and characterize the stochastic user who is unknown to the platform a priori. At each round $t$ of interaction between the platform and the stochastic user, the following happens:
\begin{itemize}
    \item The platform selects an arm (e.g., shows an ad) denoted by $J_t$ (say, $J_t=j$).
    \item The user/environment generates a sample $R_{t,j}\,\sim\,\xi_j$, as well as samples from the wear-in and wear-out distributions: $D_{t,j}\,\sim\,\xi_j^D$ and $Z_{t,j}\,\sim\,\xi_j^Z$.
    \item The user returns the following derived sample reward $X_{t,j}$ via a click/purchase given by:
\begin{equation}\label{eq:feedback_general}
    X_{t,j} =  R_{t,j} \indicator\left[ Z_{t,j} \geq f_{t,j}(N) \geq D_{t,j} \right],
\end{equation}
which is the only quantity observed by the platform. Here $\indicator[\,]$ is the indicator function, and $f_{t,j}(N)$ is a \emph{history function} that encapsulates contribution of prior user interactions, i.e., events and outcomes of rounds $\{t-N,...t-1\}$ that capture the priming effect. 
\end{itemize}

We assume that the value of $f_{t,j}(N)$ only depends on the number of times arm $j$ was played by the platform in the past $N$ rounds and is independent of the other choices made, and is non-decreasing in $N$. We also assume that each play of arm $j$ in the relevant history contributes equally to the value of $f_{t,j}(N)$. With these assumptions, we focus our analysis on the following reward accrual model:  
\begin{equation}\label{eq:feedback}
\resizebox{.9\columnwidth}{!} 
{
    $X_{t,j} =  R_{t,j} \indicator\left[ Z_{t,j} \geq\left(\sum_{k= \max(t-N,0)} ^t \indicator[ J_k = j]\right) \geq D_{t,j} \right].$
}
\end{equation}

\noindent\textit{Justifying the reward model:} The reward model, as a product of $R_{t,j}$ and an indicator function of the history of arms played along with the wear-in/wear-out effects, is quite practical while being amenable to analysis, and can be viewed as a stepping stone for more realistic models in the future. In fact, when $\xi_j$ are Bernoulli, as is the case with clicks and checkouts in applications such as e-commerce, there is very little loss in expressivity when using an indicator function as a multiplier versus any other continuous unimodal function of the history. One could imagine applications where a more sensitive function of history\footnote{It is possible to extend our analysis to case when different intervals of the past $N$ rounds have a weighted contribution in the definition of $f_{t,j}(N)$.} (for example, $X_{t,j}$ gradually increases due to wear-in and then decreases due to wear-out) could be relevant, where such an indicator function (even as an approximation) may be oversimplified. In these cases, a different algorithmic approach and analysis will be needed, potentially relying on RL techniques (we are already capturing the impact of past actions on the current reward without relying on RL methodology here). Further note that the priming effect in Equation \ref{eq:feedback_general} entails the necessity of having non-overlapping support for the distributions $\{\xi_j^D\}$ and $\{\xi_j^D\}$ (i.e., $a<b$). If not, there will be problem instances where no rewards would be accrued for any policy.

\noindent\textit{The stochastic nature of user behavior:} We are in a (non-contextual) stochastic bandit setting. As noted earlier, every time an arm $j$ is presented, the user generates the three random variables (and using N and $f_{t,j}(N)$) responds by giving back a transformed reward $X_{t,j}$. In doing so, they are agnostic to the strategy of the platform. Making the priming effect stochastic in each round captures a natural time-varying behavior of users on such advertising/recommendation platforms. It is important to distinguish between our setup and a setting where $D_{t,j}$ and $Z_{t,j}$ do not depend on time. The latter is a restricted stochastic setting where the priming effect is static across time ($R_{t,j}$ is the only randomness in the environment) and can result in potentially simpler learning strategies (such as playing each arm for $O(\expdj)$ rounds, and then estimating $\mu_i$s), although it is unclear if there will be any improvements in terms of the regret guarantees over ours. 

\noindent\textit{Knowledge of $\expdj$:}  We assume that the platform knows the first moment $\expdj$, as well as $N$ and $T$. This may seem limiting at first, but we argue that it is fairly benign: the platform needs no additional distributional knowledge to be able to achieve sub-linear regret matching the performance of stochastic MABs (modulo additive factors, see Section~\ref{sec:wear_in}). Further knowing (just) this first moment is readily possible for a platform using observational data or from previous interactions with its users. 
For simplicity of exposition, we assume that $\xi^D_j = \xi^D$ and $\xi^Z_j = \xi^Z$ (and thus work with $\expd$ and $\expz$ moving forward).

\begin{figure}
\centering
\includegraphics[width=.45\textwidth]{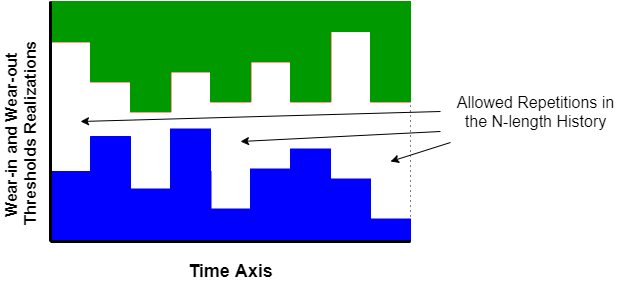}
\caption{Model Illustration: At any round (x axis), if the cumulative number of times a chosen arm was played is in the white region, then a reward is obtained by the platform. If it is lower (wear-in, blue region) or higher (wear-out, green region), the platform does not get the reward. Each pair of blue and green bars is for the arm that was pulled in that time step.}
\end{figure}\label{fig:model}

\noindent{\textbf{Goal:}} We want to design an online algorithm for the platform that plays a sequence of arms $\{J_t\}$ such that the expected (pseudo-)regret $R_T$ of the algorithm when compared to a benchmark policy is sublinear in the time horizon $T$:
\begin{equation}\label{eq:regret_general}
\begin{split}
    R_T &= \mathbb{E}\left[\sum_{t=1}^T X_{t,\pi_t} \right] - \mathbb{E}\left[ \sum_{t=1}^{T} X_{t,J_t}\right],
\end{split}
\end{equation}
where $\pi = (\pi_1,...,\pi_T)$ lies in the class of benchmark policies $\Pi_B$, and each $\pi_t$ is a function that maps the history of arms played in the past and their rewards to arm choice in round $t$.  Note that the rewards for the benchmark are censored according to that policy and not the policy the learner is playing.
In the standard bandit setting, it is standard to assume that $\pi$ is a set of constant functions (i.e., $\pi_1=...=\pi_T$) that are determined by the mean rewards $\mu_j$ for $j\,\in\,\arms$. Similarly, in our setting, we consider benchmark policies that are designed to take into account the priming effect. Intuitively, the benchmark policy of playing the best arm in hindsight may not be suitable anymore, especially when both wear-in and wear-out effects are present (Section \ref{sec:wiwo}). We defer further discussion on the choice of the benchmark policy to Sections~\ref{sec:wear_in} and~\ref{sec:wiwo}.

\section{WEAR-IN EFFECT}\label{sec:wear_in} 

We start with those instances where wear-out effect is non-existent, i.e., $Z_{t,j} = N$ for all $t$ and $j \in \arms$. This simplifies the treatment, while allowing for the same analytical tools to be extended to the general setting in Section~\ref{sec:wiwo}. 

Methods such as UCB1~\citep{auer2002finite}, MOSS~\citep{audibert2009minimax} or Thompson Sampling~\citep{chapelle2011empirical} will not succeed in our problem setting for any reasonable values of $\expd$ (and $\expz$).  For instance, the UCB1 algorithm, which is based on the optimism principle, adaptively decreases its optimism over mean rewards such that arms are easily distinguishable from each other. But while doing this, it does not allow for a direct control on how arms switch between rounds. Even when only wear-in effect is present, is impact is minimal only when $\expd << N$, which is quite limiting. To control the switching between arms, one could employ algorithms such as \aae~\citep{even2006action}, 
\ucbr~\citep{auer2010ucb} or \ucbtwo~\citep{auer2002finite}, which either play arms in a predictable round robin fashion or play the same arm consecutively. And this is precisely what we attempt to do here. 

Our proposed strategy, \bie{} (Algorithm~\ref{alg:expectation_known}) plays arms in phases (indexed by $m$). These phases are distinct and non-overlapping. The algorithm maintains a set of active arms, and in every phase, each arm from the active set is played repeatedly and consecutively. The algorithm also maintains a confidence bound $\tilde{\Delta}_m$ on the estimates of mean rewards in each phase. At the end of each phase, arms are eliminated based on confidence gaps computed using this arm-agnostic bound. We refer to the number of rounds that each active arm is played in a given phase as the \emph{incremental phase length}. The total length of the phase is the sum of such individual incremental phase lengths. In particular, a sequence of (cumulative) phase lengths $\{n_m| m=0,1,2,...\}$ determine the number of rounds each active arm has been played by phase $m$ (thus, the incremental phase length is $n_m - n_{m-1}$). And we denote the set of active arms in phase $m$ using the set $\armsm$.

Intuitively, longer incremental phase lengths help in negating the wear-in effect, whereas shorter incremental phase lengths help in negating the wear-out effect. These opposing consequences necessitate a different algorithm design when both the priming effect are present (see Section~\ref{sec:wiwo}). In the current setting, switching the arms too often reduces the rewards accumulated due to the wear-in effect and impedes the algorithm from exploring as well as exploiting what it has learned so far. Thus, a careful design of phase length is necessary, which we discuss below. In Algorithm~\ref{alg:expectation_known}, we use $T_j(m)$ to refer to the collection of times when the $j^{th}$ arm is played up to phase $m$.  Further, the estimated mean reward for arm $j$ at the end of phase $m$ is denoted by $\estmeanmj$.

\begin{algorithm}[t]
\SetAlgoLined
\textbf{Input:} A set of arms $\arms$, time horizon $T$, and phase length parameters $\{n_m| m=0,1,2,...\}$.\\
\textbf{Initialization:} Phase index $m=1$, $\armsm = \arms$, $\tilde{\Delta}_{1}$ = 1, $T_j(0) = \phi\;\; \forall j \in \arms$, where $\phi$ is the empty set, and time index $t= 1$.\\
\While{$t\leq T$ }{
\If{ $|\armsm|>1$}{
\emph{Play Arms:}\\
\For{ each active arm $j$ in $\armsm$}{
Set $T_j(m) = T_j(m-1)$.\\
Play $j$ for $n_m-n_{m-1}$ consecutive rounds and update $T_j(m)$.\\ 
\emph{Accrue rewards according to the environment model Equation (\ref{eq:feedback}).}\\
}
\emph{Eliminate Sub-optimal Arms:}\\
\For{each active arm $j$ in $\armsm$}{
$\quad \estmeanmj = \frac{1}{|T_j(m)|}\sum_{s\in T_j(m)} X_{s,j}$.
}
Construct $\armsmp$ by eliminating arms $j$ in $\armsm$ for which:\\
$\estmeanmj + \tilde{\Delta}_m/2  < \max_{j' \in \armsm} \overline{X}_{m,j'} - \tilde{\Delta}_m/2$.\\

\emph{Update the Confidence Bound:}\\
Set $\tilde{\Delta}_{m+1} = \frac{\tilde{\Delta}_m}{2}$.\\
Increment phase index $m$ by $1$ and update $t$ based on $\{n_m\}$ values up to the current phase.
}
Play the single arm in $\armsm$ and update $t$.\\
}
\caption{\bie}\label{alg:expectation_known}
\end{algorithm}

\noindent{\textbf{Benchmark Policy:}} To bound the regret defined in Equation~(\ref{eq:regret_general}) for this algorithm, we consider the benchmark policy $\pi$ to be one which plays the arm with the highest mean reward at all rounds. It can be shown that such a policy is optimal even with the wear-in effect (See Appendix \ref{subsec:optimal benchmark} for a proof).

There are two key aspects to bounding regret with respect to the aforementioned benchmark for Algorithm~\ref{alg:expectation_known}: (a) identifying an appropriate $n_m$ that depends on the wear-in effect parameter $\expd$, and (b) showing that this $n_m$ swiftly eliminates the sub-optimal arms. With these two steps addressed, we can get a regret guarantee such as below. 
\begin{theorem}\label{thm:expectation_known} For any $\lambda >0$, the expected (pseudo-)regret of \bie{} (Algorithm~\ref{alg:expectation_known}) is bounded as:
\begin{flalign}
 R_T \leq  & \underset{i\in\mathcal{K}_1}{\sum}\left(\Delta_i + \frac{64 \log(T)}{\Delta_i} + \frac{64\log(T)}{3}\right.&\nonumber\\
 &\left.\quad\quad\quad+ 32\sqrt{\log\left(\frac{4}{\Delta_i}\right)\expd\log(T)} \right)&\nonumber\\
 &\quad +\underset{i\in \mathcal{K}_1}{\sum} \frac{4\Delta_i}{T} + \underset{i\in \mathcal{K}_2}{\sum} \frac{32}{T} + \underset{\{i \in \mathcal{K}_2: \Delta_i < \lambda\}}{\max} \Delta_i T,&\nonumber
\end{flalign}
where $\mathcal{K}_1 = \{ i\in \mathcal{K} \vert \Delta_i > \lambda \}$, $\mathcal{K}_2 = \{ i\in \mathcal{K} \vert \Delta_i >0 \}$, and  $\Delta_{i} = \mu^* - \mu_{i}$.
\end{theorem}

Given an appropriate choice for phase lengths $n_m$, the proof of the above theorem follows that in~\cite{auer2010ucb}, where a similar phased-based algorithm was suggested for the vanilla stochastic MAB. Following Theorem~\ref{thm:expectation_known}, we can also obtain a corresponding instance independent bound, as shown in the following corollary.
\begin{corollary}\label{coroll:expectation_known}
For all $T \geq K$, choosing $\lambda=\sqrt{\frac{K\log(T)}{T}}$ and using $\log(1/\tilde{\Delta}_m) \leq \log(T)$, the expected (pseudo-)regret of \bie{} is $\mathrm{O}\left(\sqrt{KT\log T} + K\sqrt{\log^2 T \expd}\right)$.
\end{corollary}

A key point to note is that the \textit{wear-in} effect parameter $\expd$ appears as an additive penalty. The leading term, $\sqrt{KT\log T}$ is only a logarithmic factor away from the best known bounds for the vanilla stochastic MAB~\citep{audibert2009minimax}. Thus our regret upper bound behaves gracefully with the level of the wear-in priming effect. 

To define suitable phase length parameters $\{n_m\}$, we design martingales on the sequences of bias adjusted rewards sequences and eventually bound the growth of such martingales under our model. Appendix~\ref{subsec:prelim} contains an overview of martingales, stopping times, and key concentration bounds required for the analysis. For a detailed discussion on martingale properties, one can refer~\cite{mitzenmacher2005probability}. 
For any active arm $j$ and phase $m$, let $S_{m,j}$ denote the time in this phase when the algorithm starts playing this arm. Similarly let $U_{m,j}$ denote the time in this phase when the algorithm stops playing this arm. Also, let $T_t(j,N)$ be a random variable that denotes the number of times that the arm $j$ was played in the rounds $\{t-N,...t-1\}$.  We define a filtration $\{\mathcal{G}_s\}_{s=0}^{\infty}$ by setting $\{\mathcal{G}_0\} = \{\Omega, \phi \}$ with $\Omega$ suitably defined, and letting $\{\mathcal{G}_t\}$ to be the $\sigma$-algebra over $(X_{1}....X_t,J_1....J_t,D_{1,J_1}....D_{t,J_t},R_{1,J_1}...R_{t,J_t})$. In Lemma~\ref{stocnm} below, we give a constructive proof for the choice of $n_m$ such that the estimated mean reward for an arm $j$ gets closer to its true mean at the end of phase, and we can use this property to eliminate sub-optimal arms quickly.

\begin{lemma}\label{stocnm}
There exists a positive $n_m$ for which the estimate $\overline{X}_{m,j}$ calculated by Algorithm~\ref{alg:expectation_known} for an active arm $j$ ($j\in \mathcal{K}_m$) and phase $m$, satisfies $\overline{X}_{m,j} - \mu_j \leq \tilde{\Delta}_m/2$ with probability at least $1-\frac{2}{T^2}$.
\end{lemma}

Below, we show how unlikely it is to grossly overestimate the mean value, assuming $j$ is a sub-optimal arm.

\noindent{\textbf{Outline of the proof:}} We build on the observation that the cumulative sums of bias adjusted rewards ($\overline{X}_{m,j} - \mu_j$) can be decomposed into a couple of martingale sequences (see the first two terms in Equation~\ref{eq:71}). As the algorithm progresses, we show via Lemmas~\ref{term1} and~\ref{term2} that the growth of both these martingales can be bounded with high probability in terms of the phase length $n_m$ and the wear-in effect parameter $\expd$. To start, it follows that for each arm $j$:
\begin{gather}
    \sum_{i=1}^m\sum_{t=S_{i,j}}^{U_{i,j}} ( X_{t,j} - \mu_j ) \leq \sum_{i=1}^m\sum_{t=S_{i,j}}^{U_{i,j}} ( R_{t,J_t} - \mu_j )\nonumber \\-  \sum_{i=1}^m\sum_{t=S_{i,j}}^{U_{i,j}} R_{t,J_t}\indicator\{ T_t(J_t,N) \leq D_{t,J_t}\}.\label{eq:stoc1}
\end{gather}
Since only one arm is played at a time, therefore, $J_t=j$ and $T_t(J_t,N) \leq t-S_{i,j}$ within a phase. Define $A_{i,t} := R_{t,J_t}\indicator\{ t \leq S_{i,j} + d_{t,J_t}\}$ and $M_t := \sum_{i=0}^m A_{i,t}\indicator\{ S_{i,j}\leq t \leq U_{i,j} \}$ . We can upper bound Equation \ref{eq:stoc1} and write it in terms of $M_t$ as:
\begin{gather}
    \sum_{i=1}^m\sum_{t=S_{i,j}}^{U_{i,j}} ( X_t - \mu_j ) \leq \sum_{i=1}^m\sum_{t=S_{i,j}}^{U_{i,j}} ( R_{t,j} - \mu_j )\nonumber\\ +\sum_{t=1}^{U_{m,j}} (\mathbb{E}[M_t\vert G_{t-1}] - M_t ) -
    \sum_{t=1}^{U_{m,j}}\mathbb{E}[M_t\vert G_{t-1}].\label{eq:71}
\end{gather}
Due to the above construction, we are able to succinctly separate the loss/modulation in rewards due to the \textit{wear-in} effect. Next, we bound each term in Equation \ref{eq:71} individually. The first term is the deviations of the un-modulated rewards from their true means, and a reasonable upper bound is desired to get the right dependence on $T$ in the Corollary~\ref{coroll:expectation_known}. Next, the second term captures the impact of the wear-in effect, and can be upper bounded using Lemmas~\ref{term1} and~\ref{term2}. Finally, by taking trivial non-negative upper bound on the last term in Equation~\ref{eq:71} above, and applying a simple union bound, we obtain the following expression for $n_m$ that guarantees the claim made in the statement of Lemma~\ref{stocnm} with probability at least $1-\frac{2}{T^2}$:
\begin{equation}
\label{eq:stocnmval}
\resizebox{.9\columnwidth}{!} 
{
    $n_m \leq 1+\frac{4 \log(T)}{\tilde{\Delta}^2_m} + \frac{16\log(T)}{3\tilde{\Delta}_m} + \frac{8\sqrt{m\expd\log(T)}}{\tilde{\Delta}_m}.$
}
\end{equation}
Below, we now discuss the supporting lemmas needed for Lemma~\ref{stocnm}. First, we start with the second term of Equation~\ref{eq:71}. The following lemma shows that $Y_s = \sum_{t=1}^s (\mathbb{E}[M_t\vert G_{t-1}] - M_t)$ forms a martingale. This fact is a prerequisite for Lemma~\ref{term2}.
\begin{lemma}\label{martingaleproof}
$Y_s := \sum_{t=1}^s (\mathbb{E}[M_t\vert G_{t-1}] - M_t)$ for all $s\geq 1$ with $Y_0=0$ is a martingale with respect to the filtration $\{G_s\}^{\infty}_{s=0}$ with increments $C_s = \mathbb{E}[M_s\vert G_{s-1}] - M_s $ satisfying $\mathbb{E}[C_s\vert G_{s-1}] = 0$ and  $C_s \leq 1$ for all $s\geq 1$.
\end{lemma}

The next lemma is used to bound the sum of deviations of the received rewards from their mean values at the end of a phase (the first term in Equation~\ref{eq:71}). In effect, we show that the sum of accrued rewards will be close to their sum of means with high probability. This is a high probability guarantee that reasonably long sequences of rewards can give information about the mean reward parameters of the arms.
\begin{lemma}\label{term1} 
With probability at least $1 - \frac{1}{T^2}$, $\sum_{i=1}^m\sum_{t=S_{i,j}}^{U_{m,j}}( R_{t,j} -\mu_j ) \leq \sqrt{n_m \log(T)}.$
\end{lemma}
\noindent{\textbf{Outline of the proof:}}
For an arm $j$ that was played at time $t$, the quantity $R_{t,J_t} -\mu_j$ can be either positive or negative. And we are interested in bounding the maximum cumulative positive growth. One can interpret this bound to be the maximum expected value of a finite $1$-dimensional random walk, where an agent is taking steps at $n_m$ randomly chosen time instances in the given horizon $T$. The cumulative displacements of the agent because of the random walk form a martingale with an appropriately defined filtration. We use a modified version of the Azuma-Hoeffding inequality (see Lemma~\ref{doob} in Appendix~\ref{subsec:prelim}) to upper bound this displacement with high probability.  

Next, through the following lemma, we show that the wear-in effect martingale sequence (the second term in Equation~\ref{eq:71} does not have any \textit{sudden jumps} by upper bounding the corresponding quadratic variation process.

\begin{lemma}\label{variation}
For any $t$, let $P_t= \mathbb{E}[M_t\vert G_{t-1}] - M_t$.  It follows that: $\sum_{t=1}^{U_{m,j}} \mathbb{E}[P_t^2\vert G_{t-1}] \leq m\expd.$
\end{lemma}

\noindent{\textbf{Outline of the proof:}} From the definition of $P_t$, it is easy to see that $\sum_{t=1}^{U_{m,j}} \mathbb{E}[P_t^2\vert G_{t-1}] \leq \sum_{t=1}^{U_{m,j}} \mathbb{E}[M^2_t\vert G_{t-1}]$. Recall that $M_t$ denotes the sum of $m$ random variables of which only one is non-negative while others are zero, and this depends on the phase number the time index $t$ belongs to. Intuitively, this contributes the factor  $m$ to bound claimed. In a phase $i$, for arm $j$, the reward is lost if $\indicator\{t < S_{i,j}+D_{t,j}\} = 1$. We are interested in sum of these indicator values when $t$ varies from $1$ to $T$. In expectation, each indicator is given by $\mathbb{P}(t < S_{i,j} + D_{t,j})$, the summation of which can be upper bounded by $\expd$.

Finally, we can now bound the second term in Equation~\ref{eq:71}, using the following lemma.

\begin{lemma}\label{term2}
With probability at least $1-\frac{1}{T^2}$, $\sum_{t=1}^{U_{m,j}}(\mathbb{E}[M_t\vert G_{t-1}] - M_t ) \leq \frac{2}{3}\log(T)+ \sqrt{\frac{4\log^2(T)}{9}+4m\expd\log(T)}.$
\end{lemma}
\noindent{\textbf{Outline of the proof:}}
It follows from Lemma~\ref{martingaleproof} that $Y_{U_{m,j}}= \sum_{t=1}^{U_{m,j}}(\mathbb{E}[M_t\vert G_{t-1}] - M_t )$ forms a martingale, and represents an upper bound on the missed rewards due the wear-in effect. Next, note that each reward is bounded. Further, we had claimed in Lemma~\ref{variation} that there are no sudden jumps in the growth of the corresponding martingale sequences. Using Lemma~\ref{freedman}(see Appendix~\ref{subsec:prelim} for the statement) implies that $Y_{U_{m,j}}$ is bounded as well, which can then be used to show that the above inequality is true with high probability. Detailed proofs of Lemmas~\ref{stocnm}-\ref{term2} and Theorem~\ref{thm:expectation_known} are provided in Appendix~\ref{subsec:missing_proofs}.
\section{WEAR-IN AND WEAR-OUT EFFECTS}\label{sec:wiwo}

The general setting involving both wear-in and wear-out effects is significantly harder to tackle, primarily because of the opposing nature of these two. Qualitatively, wear-in effect necessitates continued exploitation (repetition of actions), while on the other hand, wear-out penalized continued exploitation. In \bie{}, arms are consecutively repeated $n_m-n_{m-1}$ times (see Equation \ref{eq:stocnmval}) in each phase $m$. The amount of repetition quickly supersedes $N$, which leads to zero reward accrual due to the wear-out effect for a majority of the rounds (see Equation~\ref{eq:feedback}). Hence, an appropriate algorithm for the general priming situation should be able to track the number of plays of each arm like \bie{}, potentially in a phased manner, but also also have sufficient \textit{local exploration} to discourage any detrimental wear-out effect.

Recall (from Section \ref{sec:problem_definition}) that the distributions $\xi_j^D$ and $\xi_j^Z$ have supports on $\{0,...,a\}$ and $\{b,...,N\}$, with $N \in \mathbb{N}$ and $0\leq a<b \leq N$ being unknown fixed constants. Clearly, $D_{t} = 0$ or $Z_{t} = N$ would imply there are no loss of rewards due to wear-in or wear-out respectively (dropping the dependence on arm $j$ here for clarity). Additionally, $D_t > \nicefrac{N}{2}$ would make the wear-in effect $\textit{too strong}$ as it necessitates the following: any algorithm will need to repeat the same arm for a majority portion of any contiguous $N$ rounds, and further the algorithm would fail to accrue rewards from other arms that are played in the remaining portion of this set of $N$ rounds. Hence, we assume $a \leq \nicefrac{N}{2} < b$ (refer to Appendix \ref{subsec:alpha_beta_assumption} for a more detailed discussion about this assumption).

The way we tackle both the wear-in and wear-out effects is through a \textit{key observation}: that playing a pair of arms (or more) with equal probability may provide sufficient \textit{local exploration} which could nullify wear-out effect, while ensuring that arms are repeated often enough for sufficient \textit{global exploitation} to also counter wear-in effect simultaneously. This is exactly what we do in the algorithm \bieboth{} (see Algorithm \ref{alg:expectation_known2}). Similar to the algorithm \bie{} (of Section \ref{sec:wear_in}), \bieboth{} plays arms in phases. However, instead of repeating the same arm continuously for $n_m-n_{m-1}$ times, the algorithm plays a pair of arms with equal probability for a collective $n_m-n_{m-1}$ times in the $m$-th phase (for some new optimally chosen phase length parameters $\{n_m\}$). Intuitively it is similar to the following hypothetical setting: construct $^{|\mathcal{K}}|C_2$ pairs of arms from the original set $\arms$, appropriately calculate the mean rewards for each pair and run an instance of \bie{} with $^{|\mathcal{K}|}C_2$ arms. To be able to reuse the techniques and analysis from Section \ref{sec:wear_in}, we formalize the above idea of arm-pair or \textit{compound arm} next. Notation wise, let $(i,j)$ denote the compound arm constructed by playing the arms $i$ and $j$ with equal probability and $\arms^2$ denote the set of all possible pairs composed of arms in $\arms$. Clearly $\mu_{(i,j)}=\nicefrac{(\mu_i+\mu_j)}{2}$.

\begin{algorithm}[t]
\SetAlgoLined
\textbf{Input:} Compound arms composed of pairs from $\arms$: $\arms^2$, time horizon $T$, and phase length parameters $\{n_m| m=0,1,2,...\}$.\\
\textbf{Initialization:} Phase index $m=1$, $\armsm^2 =\arms^2$, $\tilde{\Delta}_{1}$ = 1, $T_{(i,j)}(0) = \phi\;\; \forall (i,j) \in \arms^2$, where $\phi$ is the empty set, and time index $t= 1$.\\
\While{$t\leq T$ }{
\If{$|\armsm^2|>1$}{
\emph{Play Compound Arms:}\\
\For{ each compound arm $(i,j)$ in $\armsm^2$}{
Set $T_{(i,j)}(m) = T_{(i,j)}(m-1)$.\\
Play either $i$ or $j$ with equal probability for $(n_m-n_{m-1})$ consecutive rounds and update $T_{(i,j)}(m)$.\\
\emph{Accrue rewards according to the environment model Equation \ref{eq:feedback}.}\\
}
\emph{Eliminate Sub-optimal Pairs:}\\
\For{each active pair $(i,j)$ in $\armsm^2$}{
\resizebox{.8\columnwidth}{!} 
{
$\quad \overline{X}_{m,(i,j)} = \frac{1}{|T_{(i,j)}(m)|}\sum_{s\in T_{(i,j)}(m)} X_{s,(i,j)}$.
}
}
Construct $\armsmp^2$ by eliminating all pairs $(i,j)$ in $\armsm^2$ for which:\\
\resizebox{.8\columnwidth}{!} 
{
$\overline{X}_{m,(i,j)} + \tilde{\Delta}_m/2  < \max_{(i,j)' \in \armsm^2} \overline{X}_{m,(i,j)'} - \tilde{\Delta}_m/2.$
}\\
\emph{Update the Confidence Bound:}\\
Set $\tilde{\Delta}_{m+1} = \frac{\tilde{\Delta}_m}{2}$.\\
Increment phase index $m$ by $1$ and update $t$ based on $\{n_m\}$ values up to the current phase.
}
Play the single compound arm in $\armsm^2$ and update $t$.\\
}
\caption{\bieboth}\label{alg:expectation_known2}
\end{algorithm}

\noindent{\textbf{Benchmark Policy:}} Again, recall the definition of expected regret from  Section \ref{sec:problem_definition} (Equation \ref{eq:regret_general}). Since any benchmark also endures the priming effect, the benchmark policy of Section \ref{sec:wear_in}, which plays the best arm consistently for all rounds is not optimal (in fact, it will have linear regret). So for the setting in this Section, we define $\pi$ to be the policy that knows the mean rewards of all arms and plays the top two arms (in terms of mean rewards, denoted by $\mu_{(1)}^* $ and $ \mu_{(2)}^*$) with equal probability in each round.  The policy $\pi$ may not necessarily be optimal, however it provides a natural performance measure to contrast against for cases where a learning algorithm plays \textit{compound arms}. For a more detailed discussion on the above benchmark and the optimal benchmark, see Appendix \ref{subsec:optimal benchmark}.

As in the previous section, we assume the knowledge of the wear-in parameter $\expd$. The key design challenge is to calculate the $\{n_m\}$ sequence such that Algorithm \ref{alg:expectation_known2} quickly eliminates sub-optimal \textit{compound arms}, which we do in Lemma \ref{lemma:stocnm_wiwo}. Using this choice for \bieboth{}, we obtain an expected regret upper bound against the above benchmark as follows.
\begin{theorem}\label{thm:wi_wo} For any $\lambda >0$, the expected (pseudo-)regret of \bieboth{} is bounded as:
\begin{flalign}
 R_T \leq  & \underset{i\in\mathcal{K}^2_1}{\sum}\left(\Delta_i + \frac{64 \log(T)}{\Delta_i} + \frac{64\log(T)}{3}\right.&\nonumber\\
 &\left.\quad\quad\quad\quad+ 32\sqrt{\log\left(\frac{4}{\Delta_i}\right)N\expd\log(T)} \right)&\nonumber\\
 &\quad +\underset{i\in \mathcal{K}^2_1}{\sum} \frac{4\Delta_i}{T} + \underset{i\in \mathcal{K}^2_2}{\sum} \frac{32}{T} + \underset{\{i \in \mathcal{K}^2_2: \Delta_i < \lambda\}}{\max} \Delta_i T,&\nonumber
\end{flalign}
where  $\mathcal{K}^2_1 = \{ (i,j)\in \mathcal{K}^2 \vert \Delta_{(i,j)} > \lambda \}$, $\mathcal{K}^2_2 = \{ (i,j)\in \mathcal{K}^2 \vert \Delta_{(i,j)} >0 \}$, and $\Delta_{(i,j)} = \frac{1}{2}(\mu_{(1)}^* + \mu_{(2)}^* - \mu_{i} - \mu_{j})$.
\end{theorem}

The proof of Theorem~\ref{thm:wi_wo} follows a similar  proof strategy as that of Theorem~\ref{thm:expectation_known}, and has been provided in detail in Appendix \ref{subsec:missing_proofs}. It relies on an appropriate choice for $\{n_m\}$, as given by Equation~\ref{eq:stocnmval_wiwo} (see Appendix~\ref{subsec:missing_proofs}). As before, we can also get an instance independent bound as shown below.

\begin{corollary}\label{coroll:wi_wo}
For all $T \geq K^2$, choosing $\lambda=\sqrt{\frac{K^2\log(T)}{T}}$, using the inequality $\log(1/\tilde{\Delta}_m) \leq \log(T)$, and the observation that both $|\arms^2_1|$ and $|\arms^2_2| $ are $\mathrm{O}(K^2)$, gives the following upper bound on the expected (pseudo-)regret of \bieboth{} is $R_T \leq \mathrm{O}\left(K\sqrt{T\log T} + K^2\sqrt{\log^2 T N\expd}\right).$
\end{corollary}

When compared to Corollary~\ref{coroll:expectation_known}, we immediately notice the following. The first term has an additional factor of $\sqrt{K}$ and the second term has an additional factor of $K\sqrt{N}$. The increased dependence on $K$ can easily be attributed to the choice of using compound-arms by the algorithm. The dependence on $\sqrt{N}$ comes from the wear-out effect. Because $\expd$ can be $O(N)$ itself, overall, one can conclude that the priming effect (wear-in and wear-out) lead to an additive linear term in $N$ (which also parameterizes these effects).

The following lemma provides a recipe to calculate $n_m$ which is an input to Algorithm~\ref{alg:expectation_known2}. 

\begin{lemma}\label{lemma:stocnm_wiwo}
There exists a positive $n_m$ for which the estimate $\overline{X}_{m,(i,j)}$ calculated by Algorithm~\ref{alg:expectation_known2} for an active pair $(i,j)$ ($(i,j)\in \mathcal{K}^2_m$) in phase $m$ satisfies the following inequality with probability at least $1-\frac{2}{T^2}$:
\begin{equation*}
    \overline{X}_{m,(i,j)} - \mu_{(i,j)} \leq \tilde{\Delta}_m/2,
\end{equation*}
where $\mu_{(i,j)} = \frac{1}{2}(\mu_i + \mu_j)$.
\end{lemma}
\noindent{\textbf{Outline of the proof:}} The phases are defined with respect to the compound arms and the reward $R_{t,J_t}$ is due to arm played by Algorithm~\ref{alg:expectation_known2} at time $t$. We continue to build up on the same observation as used in the proof of Lemma~\ref{stocnm}: the cumulative sums of bias adjusted rewards ($\overline{X}_{m,(i,j)} - \mu_{(i,j)}$) can be decomposed into three sequences, where the additional third is due to the loss of rewards due to wear-out effect see Appendix \ref{subsec:missing_proofs} for details. The value of $n_m$ for which the lemma holds is given by:
\begin{equation}\label{eq:stocnmval_wiwo}
\resizebox{.85\columnwidth}{!} 
{
$n_m \leq 1+\frac{4 \log(T)}{\tilde{\Delta}^2_m} + \frac{16\log(T)}{3\tilde{\Delta}_m} + \frac{8\sqrt{Nm\expd\log(T)}}{\tilde{\Delta}_m}.$
}
\end{equation}
\vspace{-.26in}
\section{CONCLUSION}\label{sec:conclusion}

We considered the problem of showing recommendations and ads when the user behavior is influenced by priming effect: wear-in, where the user responds positively when shown the same recommendation multiple times in the recent past, and wear-out, where the user response dampens when the same recommendation is shown too frequently in the recent past. Modeling this in a bandit framework, we develop a new algorithm and show how its performance in terms of regret can be bounded. 

An open problem is to develop a theory for general temporal dependencies of current rewards on past actions and rewards. Any departure from the functional form as described in the Section~\ref{sec:problem_definition} may need additional analytical tools beyond what was presented here. One possible way is to model the temporal dependency structures via Markov decision processes. Regret analysis may not be straightforward in such settings or may produce loose bounds.

Further, the current regret bounds contain additive terms for priming, and it might be possible to achieve a better dependence on the priming effect parameters while having a worse dependence on other parameters. In this regard, characterizing the benchmarks thoroughly and finding tight lower bounds under the current model is a potential starting point. Extensions to adversarial behavior models, inclusion of contexts, handling non-stationarity, and most importantly, developing regret minimizing algorithms when multiple user behavioral effects including priming occur simultaneously, are some additional future research directions.

\bibliographystyle{plainnat}
\bibliography{impaired.bib}
\clearpage
\appendix
\section{APPENDIX}\label{sec:appendix}

\subsection{PRELIMINARIES}\label{subsec:prelim}
\begin{definition}\label{def:martingale}
(Martingale) A sequence of random variables $\{Z_i\}_{i=0}^n$ is a martingale with respect to the sequence $\{X_i\}_{i=0}^n$, if for all $n\geq 0$, the following conditions hold:
\begin{itemize}
    \item $Z_n$ is a function of $\{X_i\}_{i=0}^n$
    \item $\mathbb{E}[\vert Z_n \vert] < \infty$
    \item $\mathbb{E}[Z_n \vert X_0...X_{n-1}]  = Z_{n-1}$.
\end{itemize}
\end{definition}

\begin{definition}
(Filtration) Given a stochastic process, $\{X_t\}$ and a Borel space, $\mathcal{B}:= (T,\Sigma)$, then the sequence of nested $\sigma-$algebras, $\mathcal{F}_0 \subseteq \mathcal{F}_1 \subseteq \mathcal{F}_2... \subseteq \mathcal{F}_t \, \mathcal{F}_i \in \Sigma$ which may contain contain some information about $\{X_t\}$ is called a filtration. The stochastic process $\{X_t\}$ is said to be adapted to the filtration $\{\mathcal{F}_t\}$, if $X_t$ is $\mathcal{F}_t-$measurable for all $t$. 
\end{definition}

\begin{remark}
A martingale can be equivalently defined over a filtration sequence: $Z_n$ is $\\ \mathcal{F}_n-$measurable for all $n \geq 0$ and $\mathbb{E}[Z_n \vert \mathcal{F}_{n-1} ]  = Z_{n-1}$ in the definition~\ref{def:martingale}. 
\end{remark}

\begin{definition}
(Stopping Time~\citep{mitzenmacher2005probability}) A nonnegative, integer-valued random variable $T$ is stopping time for the martingale sequence $\{Z_i\}_{i=0}^n$ if the event $T = n$ depends only on the value of the random variables, $Z_0,Z_1...Z_n$.
\end{definition}

The following lemma is Freedman's version of Bernstein inequality for martingales. Given bounded increments in the martingale sequence and with a known bound on the total conditional variation, the lemma provides a strong bounds on the value of each element in the sequence.
\begin{lemma}\label{freedman}
\textbf{Generalized Bernstein inequality for Martingales} (Theorem 1.6 in~\cite{freedman1975tail}, Theorem 10 in~\cite{pike2018bandits}) Let $\{Y_k\}_{k=0}^{\infty}$ be a real valued martingale with respect to the filtration, $\{\mathcal{F}_k\}_{k=0}^{\infty}$ with increments $\{Z_k\}_{k=1}^{\infty}$, implying $\mathbb{E}[Z_k \vert \mathcal{F}_{k-1}] = 0$ and $Z_k = Y_k - Y_{k-1}$ for $k=1,2,\hdots$. Given the martingale difference sequence is uniformly upper bounded as, $Z_k \leq b$ for $k=1,2,...$. Define the predictable variation process $W_k = \sum_{j=1}^k \mathbb{E}[Z_j^2 \vert \mathcal{F}_{j-1}]$ for $k=1,2,...$. Then for all $\alpha \geq 0$, $\sigma^2 \geq 0$, the following probability is bounded:
\begin{equation}\label{eq:concentration}
    \mathbb{P}\left( \exists k \; :\; Y_k \geq \alpha \; \textrm{ and } W_k \leq \sigma^2 \right) \leq \exp\left(-\frac{\alpha^2/2}{\sigma^2+b\alpha/3}\right).
\end{equation}
\end{lemma}
The way to interpret Lemma~\ref{freedman} is that $\alpha$ denotes a deterministic boundary that the random walk, $Y_k$ is unlikely to cross. Following lemma relates this idea to the more applicable concept of the stopping times.
\begin{lemma}\label{equivalence}
\textbf{Equivalence Principle} (Proposition 1 in~\cite{zhao2016adaptive}) For any $\delta > 0$, $\mathbb{P}( S_J \geq f(J) ) \leq \delta$ for any stopping time $J$ if and only if, $\mathbb{P}( \{\exists n,\,S_n\geq f(n)\}) \leq \delta$, where $S_J$ is a random walk.
\end{lemma}
Using the above lemma in the setting of Lemma~\ref{freedman}, we are guaranteed that $Y_J$ also follows the same concentration as given by Equation~\ref{eq:concentration}. The following lemma combines the intuition of Doob's Optional Stopping theorem with the Azuma-Hoeffding's Inequality and also bounds the growth rate of certain martingale sequences.

\begin{lemma}\label{doob}
(Lemma A.1 in~\cite{szita2011agnostic}, Lemma 11 in~\cite{pike2018bandits}) Fix the positive integers $m$, $n$ and let $a,c \in \mathbb{R}$. Let $\mathcal{F}= \{\mathcal{F}_t\}_{t=0}^{\infty}$ be a filtration, $(\rho_t)_{t=1,2,3...n}$ be $\{0,1\}$-valued and $\mathcal{F}_{t-1}$-measurable random variables, and $(Z_t)_{t=1,2,3...n}$ be $\mathcal{F}_t$-measurable $\mathbb{R}$-valued random variables satisfying $\mathbb{E}[Z_t\vert\mathcal{F}_{t-1}]=0$, $Z_t\in[a,a+c]$ and $\sum_{s=1}^n\rho_s \leq m$ with probability one. Then for any $\eta >0$:
\begin{equation*}
    \mathbb{P} \left( \sum_{t=1}^n \rho_t Z_t \geq \eta \right) \leq \exp\left(-\frac{2\eta^2}{c^2m} \right ). 
\end{equation*}
\end{lemma}

\subsection{ADDITIONAL PROOFS}\label{subsec:missing_proofs}
\begin{customlemma}{1}\label{stocnm_app}
There exists a positive $n_m$ for which the estimate $\overline{X}_{m,j}$ calculated by Algorithm~\ref{alg:expectation_known} for an active arm $j$ ( $j\in \mathcal{K}_m$ ) and phase $m$, satisfies the following inequality with probability at least $1-\frac{2}{T^2}$:
\begin{equation*}
    \overline{X}_{m,j} - \mu_j \leq \tilde{\Delta}_m/2.
\end{equation*}
\end{customlemma}
\begin{proof}
Using the above notation, it follows that for each arm $j$:
\begin{gather}
    \sum_{i=1}^m\sum_{t=S_{i,j}}^{U_{i,j}} ( X_t - \mu_j ) \leq \sum_{i=1}^m\sum_{t=S_{i,j}}^{U_{i,j}} ( R_{t,J_t} - \mu_j )\nonumber  \\- \sum_{i=1}^m\sum_{t=S_{i,j}}^{U_{i,j}} R_{t,J_t}\indicator\{ T_j(J_t,N) \leq D_{t,J_t}\}.\label{eq:stoc1_app}
\end{gather}
Since only one arm is played at a time, $T_t(J_t,N) \leq t-S_{i,j}$ and $J_t=j$ within a phase. Therefore:
\begin{gather*}
    \sum_{i=1}^m\sum_{t=S_{i,j}}^{U_{i,j}} ( X_t - \mu_j ) \leq \sum_{i=1}^m\sum_{t=S_{i,j}}^{U_{i,j}} ( R_{t,j} - \mu_j )\nonumber \\-  \sum_{i=1}^m\sum_{t=S_{i,j}}^{U_{i,j}} R_{t,j}\indicator\{ T_t(J_t,N) \leq D_{t,J_t}\}.
\end{gather*}
Define $A_{i,t} := R_{t,j}\indicator\{ t \leq S_{i,j} + D_{t,J_t}\}$ and $M_t := \sum_{i=0}^m A_{i,t}\indicator\{ S_{i,j}\leq t \leq U_{i,j} \}$ . We rewrite~(\ref{eq:stoc1}) in terms of $M_t$ as:
\begin{gather*}
    \sum_{i=1}^m\sum_{t=S_{i,j}}^{U_{i,j}} ( X_t - \mu_j ) \leq \sum_{i=1}^m\sum_{t=S_{i,j}}^{U_{i,j}} ( R_{t,j} - \mu_j ) - \sum_{t=1}^{U_{m,j}} M_t,
\end{gather*}
therefore:
\begin{gather}
    \sum_{i=1}^m\sum_{t=S_{i,j}}^{U_{i,j}} ( X_t - \mu_j ) \leq \sum_{i=1}^m\sum_{t=S_{i,j}}^{U_{i,j}} ( R_{t,j} - \mu_j )\nonumber\\ +\sum_{t=1}^{U_{m,j}} (\mathbb{E}[M_t\vert G_{t-1}] - N_t )  -
    \sum_{t=1}^{U_{m,j}}\mathbb{E}[M_t\vert G_{t-1}].\label{eq:71_app}
\end{gather}
Due to the above construction we are able to succinctly separate loss in rewards due to the \textit{wear-in} effect. We bound each term individually in Equation~\ref{eq:71_app}. The first term is nothing but the deviations of the rewards from their true means and hence a reasonable upper bound is crucial for the $T$ dependence of the leading term in Corollary~\ref{coroll:expectation_known}. On the other hand the upper bounds on the second term decides the impact of the priming effect. From Lemmas~\ref{term1} and~\ref{term2} accompanied by a trivial non-negative upper bound for the last term in Equation~\ref{eq:71_app} above, we can write that with probability at least $1-\frac{2}{T^2}$ (from union bound):
\begin{gather*}
    \sum_{i=1}^m\sum_{t=S_{i,j}}^{U_{i,j}} ( X_t - \mu_j ) \leq \sqrt{n_m \log(T)} + \frac{2}{3}\log(T) \\+ \sqrt{\frac{4\log^2(T)}{9}+4m\expd\log(T)}.
\end{gather*}
For each active arm $j \in \mathcal{K}_m$,
\begin{gather*}
    \frac{1}{n_m}\sum_{t \in T_j(m)} ( X_t - \mu_j ) \leq \sqrt{\frac{\log(T)}{n_m}} \\+ \frac{2\log(T)}{n_m} + \frac{1}{n_m}\sqrt{4m\expd\log(T)} = w_m.
\end{gather*}
Algorithm~\ref{alg:expectation_known} requires $w_m \leq \tilde{\Delta}_m/2$ so that the arm elimination condition holds good. This helps to determine the appropriate $n_m$. Let $s1=\\ \sqrt{\frac{4\log^2(T)}{9}+4m\expd\log(T)}$, then, the smallest $n_m$ which satisfies the above is given by:
\begin{gather*}
    n_m = \left\lceil\frac{1}{\tilde{\Delta}^2_m}\left( \sqrt{\log(T)}\qquad \qquad\qquad \qquad \qquad\qquad \right.\right. \\
    \quad\quad\left.\left.+ \sqrt{\log(T) + \frac{4\tilde{\Delta}_m\log(T)}{3} + 2\tilde{\Delta}_m s1 } \right )^2\right\rceil.
\end{gather*}
Using $(a+b)^2 \leq 2(a^2+b^2)$ and $x=\lceil y \rceil \Rightarrow x \leq y+1$:
\begin{gather*}
    n_m \leq  \left\lceil\frac{1}{\tilde{\Delta}^2_m}\left( 4 \log(T) + \frac{8\tilde{\Delta}_m\log(T)}{3} + 4\tilde{\Delta}_ms1\right )\right\rceil.
\end{gather*}
We can now substitute $s1$ and use inequality $\sqrt{a^2+b^2}\leq (a+b)$ to get:
\begin{gather*}
n_m \leq \left\lceil\frac{4 \log(T)}{\tilde{\Delta}^2_m} + \frac{24\log(T)}{3\tilde{\Delta}_m} + \frac{8\sqrt{m\expd\log(T)}}{\tilde{\Delta}_m} \right\rceil.
\end{gather*}
Further, we can modify the \textit{ceiling} operator with a tight upper bound as below:
\begin{gather*}
n_m \leq 1+\frac{4 \log(T)}{\tilde{\Delta}^2_m} + \frac{16\log(T)}{3\tilde{\Delta}_m} + \frac{8\sqrt{m\expd\log(T)}}{\tilde{\Delta}_m}.
\end{gather*}
This completes the proof.
\end{proof}

\begin{customlemma}{2}\label{martingaleproof_app}
$Y_s = \sum_{t=1}^s \mathbb{E}[M_t\vert G_{t-1}] - M_t)$ for all $s\geq 1$ with $Y_0=0$ is a martingale with respect to the filtration $\{G_s\}^{\infty}_{s=0}$ with increments $C_s = Y_s-Y_{s-1} = \mathbb{E}[M_s\vert G_{s-1}] - M_s )$, satisfying $\mathbb{E}[C_s\vert G_{s-1}] = 0$, $C_s \leq 1$ for all $s\geq 1$.
\end{customlemma}
\begin{proof}
To show $\{Y_s\}_{s=0}^{\infty}$ is a martingale defined on filtration $\{\mathcal{G}_s\}_{s=0}^{\infty}$, we need to show $Y_s$ is $\{\mathcal{G}_s\}$-measurable for all $s\geq 1$ and $\mathbb{E}[Y_s\vert G_{s-1}] =Y_{s-1}$.

By the definition of $\sigma$-algebra  $\{\mathcal{G}_s\}_{s=0}^{\infty}$, random variables $D_{t,J_t},R_{t,J_t}$ are all $\{\mathcal{G}_s\}$-measurable for $t\leq s$. Additionally for phases $i$ where time instance $t$ lie in phases after $i$, $\indicator\{ S_{i,j}\leq t \leq U_{i,j} \}=0$ ( measurable by $\mathcal{G}_0$ ). Hence $\{Y_s\}_{s=0}^{\infty}$ is measurable by $\{\mathcal{G}_s\}_{s=0}^{\infty}$. Now consider the conditional expectation:
\begin{gather*}
    \mathbb{E}[Y_s\vert \mathcal{G}_{s-1}] = \mathbb{E}\left[\sum_{t=1}^s (\mathbb{E}[M_t\vert G_{t-1}] - M_t) \vert \mathcal{G}_{s-1} \right ]\\
    = \mathbb{E}\left[\sum_{t=1}^{s-1} (\mathbb{E}[M_t\vert G_{t-1}] - M_t) \vert \mathcal{G}_{s-1} \right ] +\\ \mathbb{E}\left[ (\mathbb{E}[M_s\vert G_{s-1}] - M_s) \vert \mathcal{G}_{s-1} \right ]\\
    = \mathbb{E}\left[\sum_{t=1}^{s-1} (\mathbb{E}[M_t\vert G_{t-1}] - M_t) \vert \mathcal{G}_{s-1} \right ]  = Y_{s-1}.
\end{gather*}
Therefore $\{Y_s\}_{s=0}^{\infty}$ is a martingale with respect to the filtration $\{\mathcal{G}_s\}_{s=0}^{\infty}$. Clearly, the increment $C_s = Y_s - Y_{s-1} = (\mathbb{E}[M_s\vert G_{s-1}] - M_s)$ and $\mathbb{E}[C_s\vert G_{s-1}] = \mathbb{E}[(\mathbb{E}[M_s\vert G_{s-1}] - M_s)\vert G_{s-1} ] = 0$ 

Note that for any phase $i$, $A_{i,t} \leq 1$ as reward $R_{t,J_t}$ is bounded by $1$. Also, for any time $t$ $M_t \leq 1$ as  $\indicator\{ S_{i,j}\leq t \leq U_{i,j} \}$ is $1$ for only a particular phase $i$ thus $C_s \leq 1$ for $s \geq 1$. This completes the proof.
\end{proof}

\begin{customlemma}{3}\label{term1_app} With probability at least $1 - \frac{1}{T^2}$,
\begin{gather*}
     \sum_{i=1}^m\sum_{t=S_{i,j}}^{U_{m,j}}( R_{t,J_t} -\mu_j ) \leq \sqrt{n_m \log(T)}.
\end{gather*}
\end{customlemma}
\begin{proof}
We will invoke an instance of Lemma~\ref{doob} to prove the above. For arm $j$, take $n=T$, $\mathcal{F}_t$ as filtration with $\sigma$-algebra on $(X_1,....X_t,R_{1,j}...R_{t,j})_{t=1,2...T}$. Let $Z_t = R_{t,j}-\mu_j$ and $\rho_t = \indicator\{J_t=j,\;t\leq U_{m,j}\}$. Therefore $\sum_{t=1}^T\rho_t$ is nothing but number of times arm $j$ was pulled till phase $m$, which is equal to $\vert T_j(m) \vert$ by definition. Also, $\vert T_j(m) \vert \leq n_m$. Hence the summation can be alternatively written as: $\sum_{t\in T_j(m)} (R_{t,j} -\mu_j) = \sum_{t=1}^T \rho_t(R_{t,j}-\mu_j)$.

Additionally, for any $1\leq t \leq T$, $\rho_t = \indicator\{J_t=j,\;t\leq U_{m,j}\}$ is $\mathcal{F}_{t-1}$-measurable. Given all the observations $X_1,X_2....X_{t-1}$ till $(t-1)$, we know in which phase does $t$ belongs. This is because of the phased nature of Algorithm~\ref{alg:expectation_known}, thus, $\indicator\{t\leq U_{m,j}\}$ is determined. Similarly $J_t=j$ can also be determined, establishing that $\rho_t$ is $\mathcal{F}_{t-1}$-measurable. $Z_t$ is $\mathcal{F}_{t}$-measurable by definition. Taking $a=-\mu_j$ and $c=1$ in the application of Lemma~\ref{doob}, we get the stated result.
\end{proof}

\begin{customlemma}{4}\label{variation_app}
For any $t$, if $P_t= \mathbb{E}[M_t\vert G_{t-1}] - M_t$, then
\begin{gather*}
    \sum_{t=1}^{U_{m,j}} \mathbb{E}[P_t^2\vert G_{t-1}] \leq m\expd
\end{gather*}
\end{customlemma}
\begin{proof}
Consider:
\begin{align*}
\sum_{t=1}^{U_{m,j}} \mathbb{E}[P_t^2\vert G_{t-1}] = \sum_{t=1}^{U_{m,j}} \mathbb{V}[M_t^2\vert G_{t-1}] \leq \sum_{t=1}^{U_{m,j}} \mathbb{E}[M_t^2\vert G_{t-1}]\\
 = \sum_{t=1}^{U_{m,j}}\mathbb{E}\left[\left( \sum_{i=1}^m A_{i,t}\indicator \{ S_{i,j} \leq t \leq U_{i,j} \} \right)^2 \vert G_{t-1} \right].
\end{align*}
Notice that the product, $\indicator \{ S_{i,j} \leq t \leq U_{i,j} \} \indicator \{ S_{k,j} \leq t \leq U_{k,j} \} =0 $ for distinct phases $i,k \leq m$ as $t$ lies only in a specific phase. Therefore:
\begin{gather*}
\sum_{t=1}^{U_{m,j}} \mathbb{E}[P_t^2\vert G_{t-1}] \leq \sum_{t=1}^{U_{m,j}}\mathbb{E}\left[\sum_{i=1}^m A_{i,t}^2\indicator \{ S_{i,j} \leq t \leq U_{i,j} \} \vert G_{t-1} \right ]\\ 
= \sum_{i=1}^m\sum_{t=1}^{U_{m,j}}\mathbb{E}\left[\sum_{i=1}^m A_{i,t}^2\indicator \{ S_{i,j} \leq t \leq U_{i,j} \} \vert G_{t-1} \right ]\\    
= \sum_{i=1}^m\sum_{t=S_{i,j}}^{U_{m,j}}\mathbb{E}\left[ A_{i,t}^2\indicator \{ S_{i,j} \leq t \leq U_{i,j} \} \vert G_{t-1} \right ].    
\end{gather*}
As $S_{i,j}$ and $U_{i,j}$ are $G_{t-1}$-measurable and if $\indicator \{ S_{i,j} \leq t \leq U_{i,j} \}=1$, therefore:
\begin{gather*}
\leq \sum_{i=1}^m\sum_{t=S_{i,j}}^{U_{i,j}}\mathbb{E}[A_{i,t}^2\vert G_{t-1}]\\     
= \sum_{i=1}^m\sum_{t=S_{i,j}}^{U_{i,j}}\mathbb{E}[R_{t,J_t}^2\indicator\{ t \leq S_{i,j} + D_{t,J_t} \} \vert G_{t-1}]\\    
\leq \sum_{i=1}^m\sum_{t=S_{i,j}}^{U_{i,j}}\mathbb{E}[\indicator\{ t \leq S_{i,j} + D_{t,J_t} \} \vert G_{t-1}]\\
= \sum_{i=1}^m\sum_{s=0}^{\infty}\sum_{s'=s}^{\infty}\sum_{t=s}^{s'}\mathbb{E}[\indicator\{S_{i,j}=s,\;U_{i,j}=s',\; t \leq s + D_{t,J_t} \} \vert G_{t-1}]\\
= \sum_{i=1}^m\sum_{s=0}^{\infty}\sum_{s'=s}^{\infty}\sum_{t=s}^{s'}\indicator\{S_{i,j}=s,\;U_{i,j}=s'\}\sum_{t=s}^{s'}\mathbb{P}( t \leq s + D_{t,J_t} )\\   
\leq \sum_{i=1}^m\sum_{s=0}^{\infty}\sum_{s'=s}^{\infty}\sum_{t=s}^{s'}\indicator\{S_{i,j}=s,\;U_{i,j}=s'\}\sum_{l=0}^{\infty}\mathbb{P}( l \leq  D ),  
\end{gather*}
\begin{equation*}
\leq \sum_{i=1}^m \expd = m\expd.    
\end{equation*}
This completes the proof.
\end{proof}

\begin{customlemma}{5}\label{term2_app}
With probability at least $1-\frac{1}{T^2}$,
\begin{gather*}
    \sum_{t=1}^{U_{m,j}}(\mathbb{E}[M_t\vert G_{t-1}] - M_t ) < \frac{2}{3}\log(T)\\ + \sqrt{\frac{4\log^2(T)}{9}+4m\expd\log(T)}.
\end{gather*}
\end{customlemma}
\begin{proof}
$Y_s = \sum_{t=1}^s (\mathbb{E}[M_t\vert G_{t-1}] - M_t)$ for all $s\geq 1$ and $Y_0=0$ is a martingale with respect to the filtration $\{G_s\}^{\infty}_{s=0}$. Also, the increments $Z_s = Y_s-Y_{s-1} = \mathbb{E}[M_s\vert G_{s-1}] - M_s $ satisfy $\mathbb{E}[Z_s\vert G_{s-1}] = 0$ and $Z_s \leq 1$ for all $s\geq 1$. Additionally the Lemma~\ref{variation} implies $\sum_{t=1}^s\mathbb{E}[Z_t^2\vert G_{t-1}] \leq m\expd$. From Lemma~\ref{freedman}, there exists a $s$ for which $\sum_{t=1}^s (\mathbb{E}[M_t\vert G_{t-1}] - M_t)$ is bounded with high probability. Now, Lemma~\ref{equivalence} suggests that $Y_J$ concentrates well for all stopping times $J$ and hence, also for $J=U_{m,j}$.
\end{proof}

\noindent{\textbf{Proof of Theorem~\ref{thm:expectation_known}:}}
\begin{proof}
We create four mutually exclusive and exhaustive cases. We then bound the expected regret conditioned on the events of these cases. For each sub-optimal arm $i$, let $m_i := \min \{ m \vert \tilde{\Delta}_m < \frac{\Delta_i}{2}\}$, is the first phase where $\tilde{\Delta}_m < \frac{\Delta_i}{2}$. We also define $\mathcal{K}_1 = \{ i\in \mathcal{K} \vert \Delta_i > \lambda \}$. The four cases are as follows:\\

\textbf{Case (a):} \textit{Arm $i$ is not deleted in phase $m_i$ with the optimal arm $*$ in the set $\mathcal{K}_{m_i}$}.\\
The phase $m_i$ is characterized by : $w_{m_i} \leq \frac{\Delta_{m_i}}{2} \leq \frac{\Delta_i}{4}$. Let $E:=\indicator\{\overline{\mu}_i \leq \mu_i + w_{m_i}\}$ and $R:=\indicator\{\overline{\mu}_* \geq \mu_* - w_{m_i}\}$.

If the events $E$ and $R$ hold then the phase-end elimination condition of the Algorithm~\ref{alg:expectation_known} is satisfied, as:
\begin{equation*}
\overline{\mu}_i + w_{m_i} \leq \mu_i + 2w_{m_i} < \mu_i + \Delta_i - 2w_{m_i} \leq \overline{\mu}_*  - w_{m_i}.
\end{equation*}
From Lemma~\ref{stocnm}, $\mathbb{P}(E) > 1-\frac{1}{T^2}$ and $\mathbb{P}(R) > 1-\frac{1}{T^2}$ follows. In this case we are interested in regret conditional on $\{E^{\complement} \cup \R^{\complement}\}$, which via an union bound argument can be shown as:
\begin{equation*}
    R_T \leq  \sum_{i\in \mathcal{K}_1} \frac{4}{T^2} T\Delta_i \leq \sum_{i\in \mathcal{K}_1} \frac{4\Delta_i}{T}.
\end{equation*}

\textbf{Case (b)}\textit{ Arm $i$ is eliminated at the phase $m_i$ with the optimal arm $* \in \mathcal{K}_{m_i}$.}\\ 
By lemma~\ref{stocnm}, in case the sub-optimal arm $i$ is eliminated in the phase $m_i$ then the maximum number of times it is played is given by Equation~\ref{eq:stocnmval}. Additionally, we make use of $\Delta_i/4 \leq \tilde{\Delta}_m\leq \Delta_i/2$ and $m_i \leq \log_2\left(\frac{4}{\Delta_i}\right) < 2\log\left(\frac{4}{\Delta_i}\right)$. Therefore:
\begin{gather*}
n_{m_i} \leq   1 + \frac{64 \log(T)}{\Delta_i^2} + \frac{64\log(T)}{3\Delta_i} + \frac{32\sqrt{\log\left(\frac{4}{\Delta_i}\right)\expd\log(T)}}{\tilde{\Delta}_i}.
\end{gather*}
Thus,
\begin{gather*}
R_T  \leq  \sum_{i\in\mathcal{K}'}\Delta_i\left( 1 + \frac{64 \log(T)}{\Delta_i^2} + \frac{64\log(T)}{3\Delta_i}\right. \\ \left.+\frac{32\sqrt{\log\left(\frac{4}{\Delta_i}\right)\expd\log(T)}}{\tilde{\Delta}_i}\right)\\
\leq  \sum_{i\in\mathcal{K}'}\left( \Delta_i + \frac{64 \log(T)}{\Delta_i} + \frac{64\log(T)}{3} \right.\\
\left.+ 32\sqrt{\log\left(\frac{4}{\Delta_i}\right)\expd\log(T)} \right).
\end{gather*}

\textbf{Case (c)} \textit{Optimal arm $*$ deleted by some sub-optimal $i$ in the set $\mathcal{K}_2$}.\\
Now, we consider the case when the last of all the optimal arms (in case there more than one), denoted $*$, is eliminated by some sub optimal arm $i$ in $\mathcal{K}_2 = \{ i\in \mathcal{K} \vert \Delta_i >0 \}$ in some round $m_*$ (overloading the definition of $m_i$, $m_*$ is any round where $*$ is eliminated). As elimination of the optimal arm can be induced by larger number of arms (the set $\mathcal{K}''$) at the end of a phase as compared to during a phase, we only need to analyze the events at the end of a phase to upper bound regret.  This is similar in spirit to the events in \textbf{Case (d)}, as $E$ and $R$ cannot hold together with the elimination condition of the Algorithm~\ref{alg:expectation_known}. Hence the probability of this happening is again upper bounded by $\frac{4}{T^2}$ by a similar argument.

The optimal arm $*$ belonged to $K_{m_s}$ corresponding to all sub-optimal arms $s$ with $m_s < m_*$. Therefore arm $i$, which causes elimination of the optimal arm $*$, should satisfy $m_i \geq m_*$. Therefore the regret is upper bounded by:
\begin{gather*}
    R_T \leq \overset{\max_{j\in\mathcal{K}_1}m_j}{\underset{m_*=0}{\sum}} \underset{i\in \mathcal{K}_2:m_i \leq m_*}{\sum} \frac{4}{T^2}.T \underset{j\in \mathcal{K}_2:m_j\geq m_*}{\max}\Delta_j,
\end{gather*}
\begin{equation*}
    \leq \overset{\max_{j\in\mathcal{K}_1}m_j}{\underset{m_*=0}{\sum}} \underset{i\in \mathcal{K}_2:m_i \leq m_*}{\sum} \frac{4}{T}4\tilde{\Delta}_{m_*},
\end{equation*}
\begin{equation*}
    \leq \underset{i\in \mathcal{K}_2}{\sum} \underset{m_*\geq 0}{\sum} \frac{16}{T} 2^{-m_*} \leq \underset{i\in \mathcal{K}_2}{\sum}  \frac{32}{T}.
\end{equation*}

\textbf{Case (d)} \textit{Arm $i\,\in\,\mathcal{K}_2$ and $\notin \,\mathcal{K}_1$}.\\ 
Here, we account for the difference in the sets $\mathcal{K}_2$ and $\mathcal{K}_1$. The following gives an upper bound on the regret conditioned on this case:
\begin{equation*}
  R_T  \leq \underset{i \in \mathcal{K}_2: \Delta_i < \lambda}{\max} \Delta_i T .
\end{equation*}
As all the four cases are mutually exclusive and exhaustive, we thus, get the desired regret upper bound.
\end{proof}

\begin{customlemma}{6}\label{stocnm_wi_wo_app}
There exists a positive $n_m$ for which the estimate $\overline{X}_{m,(i,j)}$ calculated by Algorithm~\ref{alg:expectation_known2} for an active pair $(i,j)$ ( $(i,j)\,\in \mathcal{K}^2_m$ ) and phase $m$, satisfies the following inequality with probability at least $1-\frac{2}{T^2}$ :
\begin{gather*}
    \overline{X}_{m,(i,j)} - \mu_(i,j) \leq \tilde{\Delta}_m/2,
\end{gather*}
where $\mu_{i,j} = \frac{1}{2}(\mu_i+\mu_j)$.
\end{customlemma}
\begin{proof}
The phases are defined with respect to the compound arms and the reward $R_{t,J_t}$ is due to arm played by Algorithm~\ref{alg:expectation_known2} at time $t$. Define a filtration $\{\mathcal{G}_t\}_{t=0}^{\infty}$, with $\mathcal{G}_t$ a $\sigma$-algebra over $(X_{1}....X_t$,$J_1....J_t$,$D_{1,J_1}....D_{t,J_t}$,$Z_{1,J_1}....Z_{t,J_t}$, $R_{1,J_1}...R_{t,J_t})$. Most notations introduced for the proof of Lemma~\ref{stocnm_app} carry through unaltered here. For each compound arm $(i,j)$:
\begin{gather}
\sum_{k=1}^m\sum_{t=S_{k,(i,j)}}^{U_{k,(i,j)}} ( X_t - \mu_{(i,j)} ) \leq \sum_{k=1}^m\sum_{t=S_{k,(i,j)}}^{U_{k,(i,j)}} ( R_{t,J_t} - \mu_{(i,j)} )\nonumber \\ -\sum_{k=1}^m\sum_{t=S_{k,(i,j)}}^{U_{k,(i,j)}} R_{t,J_t}\indicator\{ \min(t-S_{k,(i,j)},T_t(J_t,N)) \leq D_{t,J_t}\}\nonumber \\
-\sum_{k=1}^m\sum_{t=S_{k,(i,j)}}^{U_{k,(i,j)}} R_{t,J_t}\indicator\{ \min(t-S_{k,(i,j)},T_t(J_t,N)) \geq Z_{t,J_t}\}\nonumber \\
\leq \sum_{k=1}^m\sum_{t=S_{k,(i,j)}}^{U_{k,(i,j)}} ( R_{t,J_t} - \mu_{(i,j)} )  -\sum_{k=1}^m\sum_{t=S_{k,(i,j)}}^{U_{k,(i,j)}} R_{t,J_t}\indicator\{ T_t(J_t,N) \leq D_{t,J_t}\}\nonumber \\
-\sum_{k=1}^m\sum_{t=S_{k,(i,j)}}^{U_{k,(i,j)}} R_{t,J_t}\indicator\{ \min(t-S_{k,(i,j)},T_t(J_t,N)) \geq Z_{t,J_t}\}\label{eq:stoc_wiwo_app},
\end{gather}
where the second term is due to the loss in rewards due to wear-in and third term accounts for the same due to wear-out. Using a non-negative upper bound for the third term, we can loosen the upper bound and write:
\begin{gather}
    \sum_{k=1}^m\sum_{t=S_{k,(i,j)}}^{U_{k,(i,j)}} ( X_t - \mu_{(i,j)} )
    \leq \sum_{k=1}^m\sum_{t=S_{k,(i,j)}}^{U_{k,(i,j)}} ( R_{t,J_t} - \mu_{(i,j)} )\nonumber \\ -\sum_{k=1}^m\sum_{t=S_{k,(i,j)}}^{U_{k,(i,j)}} R_{t,J_t}\indicator\{ T_t(J_t,N) \leq D_{t,J_t}\}\label{eq:stoc1_wiwo_app}.
\end{gather}
By ignoring the wear-out term, we loosen the upper bound however it is intuitive to see that the loss of rewards due to wear-out effect is limited. During a phase, since the algorithm plays 2 arms uniformly, it is likely that each of the two arms get played equal number of times and neither of them gets worn out in any contagious $N$ rounds. Similarly in any phase, number of such $N$ sized contagious periods where the arms play unequal number of times would be limited. Define $B_{k,t} := R_{t,J_t}\indicator\{ T_t(J_t,N) \leq D_{t,J_t}\}$ and $N_t := \sum_{k=0}^m B_{k,t}\indicator\{ S_{k,(i,j)}\leq t \leq U_{k,(i,j)} \}$. We rewrite Equation~\ref{eq:stoc1_wiwo_app} in terms of $N_t$ as:
\begin{gather*}
    \sum_{k=1}^m\sum_{t=S_{k,(i,j)}}^{U_{k,(i,j)}} ( X_t - \mu_{(i,j)} ) \leq \sum_{k=1}^m\sum_{t=S_{k,(i,j)}}^{U_{k,(i,j)}} ( R_{t,J_t} - \mu_{(i,j)} ) - \sum_{t=1}^{U_{m,j}} N_t,
\end{gather*}
therefore:
\begin{gather}
    \sum_{i=1}^m  \sum_{t=S_{i,j}}^{U_{i,j}} ( X_t - \mu_{(i,j)} ) \leq \sum_{i=1}^m\sum_{t=S_{i,j}}^{U_{i,j}} ( R_{t,J_t} - \mu_{(i,j)} )\nonumber\\ +\sum_{t=1}^{U_{m,j}} (\mathbb{E}[N_t\vert G_{t-1}] - N_t )  -
    \sum_{t=1}^{U_{m,j}}\mathbb{E}[N_t\vert G_{t-1}].\label{eq:71_wiwo_app}
\end{gather}
We bound each term individually in Equation~\ref{eq:71_wiwo_app}. The first term is nothing but the deviations of the rewards from their true means and hence a reasonable upper bound is crucial for the $T$ dependence of the leading term in Corollary~\ref{coroll:wi_wo}. On the other hand the upper bounds on the second term decides the impact of the priming effect. We again use Lemma~\ref{term1} to bound the growth of the first term in Equation~\ref{eq:71_wiwo_app}. This also works out since in the definition of compound arm $(i,j)$ the arms $i,j$ are played randomly with equal probability. Hence $R_{t,J_t} -\mu_{(i,j)}$ is still zero mean. Further we use Lemma~\ref{term2_wi_wo} accompanied by a trivial non-negative upper bound for the last two terms. Finally, we can write that with probability at least $1-\frac{2}{T^2}$ (from union bound):
\begin{flalign*}
    \sum_{k=1}^m\sum_{t=S_{k,(i,j)}}^{U_{k,(i,j)}} ( X_t - \mu_{(i,j)} ) \leq \sqrt{n_m \log(T)} + \frac{2}{3}\log(T) \\+ \sqrt{\frac{4\log^2(T)}{9}+4mN\expd\log(T)}.
\end{flalign*}
For each active arm $(i,j) \in \mathcal{K}^2_m$,
\begin{flalign*}
    \frac{1}{n_m}\sum_{t \in T_{(i,j)}(m)} ( X_t - \mu_{(i,j)} ) \leq \sqrt{\frac{\log(T)}{n_m}} \\+ \frac{2\log(T)}{n_m} + \frac{1}{n_m}\sqrt{4Nm\expd\log(T)} = w_m.
\end{flalign*}
Algorithm~\ref{alg:expectation_known} requires $w_m \leq \tilde{\Delta}_m/2$ so that the arm elimination condition holds good. This helps to determine the appropriate $n_m$. Let $s1=\\ \sqrt{\frac{4\log^2(T)}{9}+4Nm\expd\log(T)}$, then, the smallest $n_m$ which satisfies the above is given by:
\begin{flalign*}
    n_m = \left\lceil\frac{1}{\tilde{\Delta}^2_m}\left( \sqrt{\log(T)}\qquad \qquad\qquad \qquad \qquad\qquad \right.\right. \\
    \quad\quad\left.\left.+ \sqrt{\log(T) + \frac{4\tilde{\Delta}_m\log(T)}{3} + 2\tilde{\Delta}_m s1 } \right )^2\right\rceil.
\end{flalign*}
Using $(a+b)^2 \leq 2(a^2+b^2)$ and $x=\lceil y \rceil \Rightarrow x \leq y+1$:
\begin{equation*}
    n_m \leq  \left\lceil\frac{1}{\tilde{\Delta}^2_m}\left( 4 \log(T) + \frac{8\tilde{\Delta}_m\log(T)}{3} + 4\tilde{\Delta}_ms1\right )\right\rceil.
\end{equation*}
We can now substitute $s1$ and use inequality $\sqrt{a^2+b^2}\leq (a+b)$ to get:
\begin{equation*}
n_m \leq \left\lceil\frac{4 \log(T)}{\tilde{\Delta}^2_m} + \frac{24\log(T)}{3\tilde{\Delta}_m} + \frac{8\sqrt{Nm\expd\log(T)}}{\tilde{\Delta}_m} \right\rceil.
\end{equation*}
Further, we can modify the \textit{ceiling} operator with a tight upper bound as below:
\begin{equation*}
n_m \leq 1+\frac{4 \log(T)}{\tilde{\Delta}^2_m} + \frac{16\log(T)}{3\tilde{\Delta}_m} + \frac{8\sqrt{Nm\expd\log(T)}}{\tilde{\Delta}_m}.
\end{equation*}
This completes the proof.
\end{proof}
\begin{customlemma}{7}\label{martingaleproof_wi_wo}
$H_s := \sum_{t=1}^s \mathbb{E}[N_t\vert G_{t-1}] - N_t)$ for all $s\geq 1$ with $H_0=0$ is a martingale with respect to the filtration $\{G_s\}^{\infty}_{s=0}$ with increments $V_s = H_s-H_{s-1} = \mathbb{E}[N_s\vert G_{s-1}] - N_s $ satisfying $\mathbb{E}[V_s\vert G_{s-1}] = 0$ and  $V_s \leq 1$ for all $s\geq 1$.
\end{customlemma}
The proof is identical to that of Lemma~\ref{martingaleproof_app}.\\
\begin{customlemma}{8}\label{term2_wi_wo}
With probability at least $1-\frac{1}{T^2}$,
\begin{gather*}
   \sum_{t=1}^{U_{m,j}}(\mathbb{E}[N_t\vert G_{t-1}] - N_t ) < \frac{2}{3}\log(T)\\ + \sqrt{\frac{4\log^2(T)}{9}+4mN\expd\log(T)}.
\end{gather*}
\end{customlemma}
The proof is identical to that of Lemma~\ref{term2} when we use the appropriate bounds on the variation process as given by Lemma~\ref{lemma:variation_wiwo} in this section.\\

\begin{customlemma}{9}\label{lemma:variation_wiwo}
For any $t$, if $P_t= \mathbb{E}[N_t\vert G_{t-1}] - N_t$ then
\begin{gather*}
    \sum_{t=1}^{U_{m,(i,j)}} \mathbb{E}[P_t^2\vert G_{t-1}] \leq Nm\expd 
    \end{gather*}
\end{customlemma}
\begin{proof}
Consider:
\begin{gather*}
\sum_{t=1}^{U_{m,(i,j)}} \mathbb{E}[P_t^2\vert G_{t-1}] = \sum_{t=1}^{U_{m,(i,j)}} \mathbb{V}[N_t^2\vert G_{t-1}] \leq \sum_{t=1}^{U_{m,(i,j)}} \mathbb{E}[N_t^2\vert G_{t-1}]\\
 = \sum_{t=1}^{U_{m,(i,j)}}\mathbb{E}\left[\left( \sum_{k=1}^m B_{k,t}\indicator \{ S_{k,(i,j)} \leq t \leq U_{k,(i,j)} \} \right)^2 \vert G_{t-1} \right].
\end{gather*}
Notice that the product, $\indicator \{ S_{k,(i,j)} \leq t \leq U_{k,(i,j)} \} \indicator \{ S_{l,(i,j)} \leq t \leq U_{l,(i,j)} \} =0 $ for distinct phases $k,l \leq m$ as $t$ lies only in a specific phase. Therefore:
\begin{gather*}
\sum_{t=1}^{U_{m,(i,j)}} \mathbb{E}[P_t^2\vert G_{t-1}] \leq \sum_{t=1}^{U_{m,(i,j)}}\mathbb{E}\left[\sum_{k=1}^m B_{k,t}^2\indicator \{ S_{k,(i,j)} \leq t \leq U_{k,(i,j)} \} \vert G_{t-1} \right ], 
\end{gather*}
\begin{equation*}
= \sum_{t=1}^{U_{m,(i,j)}}\mathbb{E}\left[\sum_{k=1}^m B_{k,t}^2\indicator \{ S_{k,(i,j)} \leq t \leq U_{k,(i,j)} \} \vert G_{t-1} \right ],     
\end{equation*}
\begin{equation*}
= \sum_{k=1}^m\sum_{t=S_{k,(i,j)}}^{U_{m,(i,j)}}\mathbb{E}\left[ B_{k,t}^2\indicator \{ S_{k,(i,j)} \leq t \leq U_{k,(i,j)} \} \vert G_{t-1} \right ].    
\end{equation*}
As $S_{k,(i,j)}$ and $U_{k,(i,j)}$ are $G_{t-1}$-measurable and if $\indicator \{ S_{k,(i,j)} \leq t \leq U_{k,(i,j)} \}=1$, therefore:
\begin{gather*}
\leq \sum_{k=1}^m\sum_{t=S_{k,(i,j)}}^{U_{k,(i,j)}}\mathbb{E}[B_{k,t}^2\vert G_{t-1}]\\     
= \sum_{k=1}^m\sum_{t=S_{k,(i,j)}}^{U_{k,(i,j)}}\mathbb{E}[R_{t,J_t}^2\indicator\{ T_t(J_t,N) \leq D_{t,J_t} \} \vert G_{t-1}]\\    
\leq \sum_{k=1}^m\sum_{t=S_{k,(i,j)}}^{U_{k,(i,j)}}\mathbb{E}[\indicator\{ T_t(J_t,N) \leq D_{t,J_t} \} \vert G_{t-1}]\\ 
= \sum_{k=1}^m\sum_{s=0}^{\infty}\sum_{s'=s}^{\infty}\sum_{t=s}^{s'}\mathbb{E}[\indicator\{S_{k,(i,j)}=s,\;U_{k,(i,j)}=s',\; T_t(J_t,N) \leq D_{t,J_t} \} \vert G_{t-1}]\\ 
= \sum_{k=1}^m\sum_{s=0}^{\infty}\sum_{s'=s}^{\infty}\sum_{t=s}^{s'}\indicator\{S_{k,(i,j)}=s,\;U_{k,(i,j)}=s'\}\sum_{t=s}^{s'}\mathbb{P}( T_t(J_t,N) \leq D_{t,J_t} ).  
\end{gather*}
$T_t(J_t,N)$ is a random variable which takes upto $N$. Hence an union bound gives:

\begin{equation*}
\leq \sum_{i=1}^m\sum_{s=0}^{\infty}\sum_{s'=s}^{\infty}\sum_{t=s}^{s'}\indicator\{S_{i,j}=s,\;U_{i,j}=s'\}N\sum_{l=0}^{\infty}\mathbb{P}( l \leq  D ),  
\end{equation*}
\begin{equation*}
\leq \sum_{i=1}^m \expd = Nm\expd.    
\end{equation*}
This completes the proof.
\end{proof}

\textbf{Proof of Theorem~\ref{thm:wi_wo}:}\\
Proceeding as in the proof of Theorem~\ref{thm:expectation_known}, we create four mutually exclusive and exhaustive cases with the difference from the former proof being: in this setting the arm set is $\arms^2$ (set of all possible pairs of arms) which is of size $|^{\arms}C_2|$. To be able to reuse the proof for the Theorem~\ref{thm:expectation_known}, we rely on the idea of compound arms. We must replace notion of $\Delta_i$ corresponding to the arm $i$ by the analogous notions of the regret gap, $\Delta_{(i,j)}$ for the arms $(i,j)$. Similarly the sets, $\arms_1$ and $\arms_2$ would need to be appropriately redefined and would contain $\mathrm{O}(K^2)$ elements. Specifically, if $\arms^2$ be the set of all compound arms (all pair-wise combinations from $\arms$) then: $\mathcal{K}^2_1 = \{ (i,j)\in \mathcal{K}^2 \vert \Delta_{(i,j)} > \lambda \}$, $\mathcal{K}^2_2 = \{ (i,j)\in \mathcal{K}^2 \vert \Delta_{(i,j)} >0 \}$, and $\Delta_{(i,j)} = \frac{1}{2}(\mu_{(1)}^* + \mu_{(2)}^* - \mu_{i} - \mu_{j})$.

\subsection{ASSUMPTIONS ON THE SUPPORT OF THE DISTRIBUTIONS: $\xi^D$ and $\xi^Z$}\label{subsec:alpha_beta_assumption}
Recall from Section~\ref{sec:problem_definition} that the wear-in and the wear-out effects manifest through distributions $\xi^D$ and $\xi^Z$ supported on $\{0,...,a\}$ and $\{b,...,N\}$ respectively. Following assumptions intuitively follow:
\begin{itemize}
    \item $a < N$: This implies that the actions utilize the whole of relevant history duration of $N$ just to get worn-in. Otherwise, there would be problem instances when no reward is accrued at all.
    \item $b < N$: We do not allow problem instances when wear-out is weak/non-existent. As $b \geq N$ implies that there would be problem instances in which, even if the arm was played through all the $N$ past rounds, the arm does not get worn out.
    \item $a < b$ : Thus, we disallow instances where the arms get worn-out before it could be worn-in.
\end{itemize}

To analyse combined impact of the wear-in and the wear-out effects, in Section~\ref{sec:wiwo}, we made an additional assumption that $a < \nicefrac{N}{2}$. This might seem to be a strong assumption but it is practical. To see why this assumption is required consider a contradictory case: let $a > \nicefrac{N}{2}$, then we there can be instance of the distribution $\xi_j^D$ ( say $D_{t,j}$ are constant equal to $\nicefrac{N}{2}+1$) for which only meaningful arm playing strategy is to play a single for all instance in a window $N$. This is because the number of rounds required to wear-in is more than half of the history window size and hence in a window size of $N$ only one arm could be worn-in. As $b <N$ ( that is wear-out is significant in the observation window) then arm has to switch for the next window of size $N$. Thus, any algorithm in this regime must play an arm continuously for $N$ rounds wherein only $N-D-(N-Z)$ rounds accrue rewards and the algorithm must switch arms after every $N$ instances.

\subsection{OPTIMAL BENCHMARKS}\label{subsec:optimal benchmark}

\begin{lemma}\label{lemma:wi_benchmark}
Under only wear-in effect (i.e., $Z_{t,j} > N\,\forall t$ and arms $j$), the optimal constant benchmark policy ($\pi$) in the Eq (\ref{eq:regret_general}) is the one which plays the arm with the highest reward for all rounds $T$.
\end{lemma}
\begin{proof}
There is loss of rewards due to wear-in for each arm any policy may chose to play. Hence there is no gain in playing multiple arms if none of them can beat the optimal arm. Therefore, the optimal constant policy in expectation is to play the best arm for all rounds in $T$. 
\end{proof}

The above benchmark policy is sub-optimal for the setting of the Section~\ref{sec:wiwo}. Since playing the same arm continuously would make it \emph{worn-out} quickly leading to no reward being accrued. In Section~\ref{sec:wiwo}, we introduced a benchmark policy based on the notion of the \emph{compound arms}. In the following lemma we propose another benchmark policy for this setting and prove it is optimal. 

\begin{lemma}\label{lemma:wiwo_benchmark}
Under both wear-in and wear-out effects with $0<a<  \nicefrac{N}{2}<b < N$, the benchmark policy of playing top two arms, $i^*=\max_{j\,\in\,\arms}\mu_j$ and $i^{**}=\max_{j\,\in\,\arms,\,j\neq i^*}\mu_j$ alternatively is optimal. Alternatively, playing any other arm or playing in any other order cannot improve have better expected cumulative reward.
\end{lemma}
\begin{proof}
Since, just playing one arm continuously is sub-optimal, thus any non-trivial strategy would involve playing more than 1 arm. Also, as $a<\nicefrac{N}{2}$, then atmost 2 arms can be \emph{worn-in}, therefore any strategy involving more than 2 arms would lead to additional loss of rewards due to wearing-in. Finally since $b> \nicefrac{N}{2}$, therefore if the arms are being played alternatively then under no realizations of the priming effect distributions $\xi^Z$ and $\xi^D$ would the arms get \emph{worn-out}. Hence the optimal policy should play with exactly two arms.
\end{proof}

\subsection{EXPERIMENTS}\label{sec:experiments}
We run four experiments. In the first two experiments we compared our proposed algorithms \bie{} and \bieboth{} with other baseline algorithms namely \textit{AAE}~\cite{even2006action}, \textit{MOSS}~\cite{audibert2009minimax} and \textit{UCB}~\cite{auer2002finite}. In the third, we investigate the arm switching behavior of \ucb{}, showing how it can be suboptimal for certain input instances. In the fourth, we show how \bie{} performs as a function of the wear-in effect. In all these experiments, the cumulative regret curves plotted were averaged over 30 Monte Carlo runs. The bandit instances were generated randomly, unless otherwise noted.

We use a set up of $K=20$ arms and the reward distributions to be Bernoulli with randomly chosen means. $T=5000$. In the first experiment we consider only the wear-in setting with $N=10$ and $\xi^D\,\sim$ Uniform[0,N]. Figure~\ref{fig:8} shows that the standard stochastic multi armed bandit algorithms incur linear regret, whereas, \bie{} has a sub-linear regret. In the second experiment we consider both wear-in and wear-out effects with $N=10$, $\xi^D\,\sim$ Uniform[0,3] and $\xi^Z\,\sim$ Uniform[6,10]. Under this general priming setting with both the wear-in and the wear-out effects, only \bieboth{} has sub-linear regret (see Figure~\ref{fig:15}).
\begin{figure}
\centering
	\includegraphics[width=.9\columnwidth]{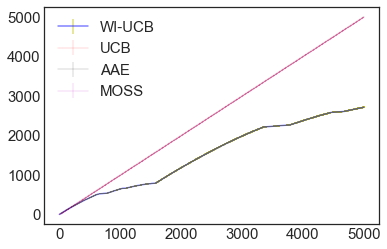}
\caption{Performance (cumulative regret) of \bie{} compared to other algorithms.}
\label{fig:8}
\end{figure}

\begin{figure}
\centering
	\includegraphics[width=.9\columnwidth]{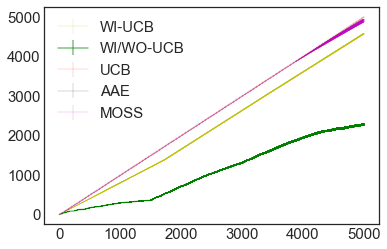}
\caption{Performance (cumulative regret) of \bieboth{} compared to other algorithms.}
\label{fig:15}
\end{figure}

We use a simple setup of $K = 30$ arms and set the reward distributions to be the Bernoulli with randomly chosen biases. The horizon length $T = 5000$. We then run \ucb under three different configurations. In the first, the bandit instance is run as is and there is an unique optimal arm. In the second, the number of optimal arms is increased to $3$, and in the third the number of optimal arms is increased to $7$. Figure~\ref{fig:7} shows the unnormalized counts of \emph{same arm plays} in the past $15$ plays. This was computed by checking how many times the current arm was also played in the past $15$ rounds. As expected, as the number of optimal arms increases, the counts of same arm plays decreases rapidly. This indicates that \ucb{} and other related algorithms may perform poorly in settings with priming.

\begin{figure}
\centering
	\includegraphics[width=.9\columnwidth]{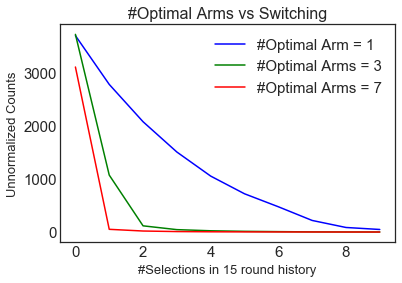}
\caption{Plot of unnormalized counts (y-axis) versus number of same arm plays in the past $15$ rounds by \ucb{} for three different settings.}
\label{fig:7}
\end{figure}

Now, we show the performance of \bie{} (Algorithm ~\ref{alg:expectation_known}) for varying levels of wear-in effect. The number of arms in this experiment is fixed at $10$. Wear-in effect is stochastic and is simulated using the absolute value normal distribution with means = $\{2,6,10,14\}$ and the standard deviation being proportional to the arm indices. The history window, $N = 20$ and the time horizon is $10000$. From Figure~\ref{fig:4}, we can observe that as the cumulative regret increases as $\expd$ is increased.
\begin{figure}
\centering
\includegraphics[width=.9\columnwidth]{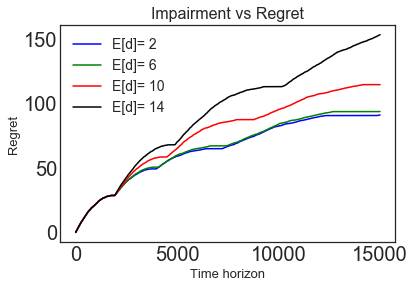}
\caption{Performance (cumulative regret) of \bie{} as the wear-in parameter is varied.}
\label{fig:4}
\end{figure}

\end{document}